\pdfoutput=1
\documentclass[11pt]{article}

\usepackage[preprint]{acl}

    \PassOptionsToPackage{numbers, compress}{natbib}

\usepackage{times}
\usepackage{latexsym}

\usepackage[T1]{fontenc}

\def\x{{\mathbf x}}
\def\y{{\boldsymbol y}}

\def\y{{\mathbf y}}

\def\X{{\boldsymbol X}}

\def\Xc{{\boldsymbol {\mathcal X}}}

\newcommand{\be}{\begin{eqnarray}}
\newcommand{\ee}{\end{eqnarray}}
\newcommand{\bi}{\begin{enumerate}}
\newcommand{\ei}{\end{enumerate}}

\usepackage[utf8]{inputenc}

\usepackage{microtype}

\usepackage{inconsolata}

\usepackage{graphicx}

\usepackage{url}            
\usepackage{booktabs}       
\usepackage{amsfonts}       
\usepackage{nicefrac}       
\usepackage{microtype}      
\usepackage{xcolor}         
\usepackage{enumitem}

\usepackage[table,xcdraw]{xcolor}

\usepackage[utf8]{inputenc} 
\usepackage[T1]{fontenc}    
\usepackage{hyperref}       
\usepackage{url}            
\usepackage{amsfonts}       
\usepackage{nicefrac}       
\usepackage{microtype}      
\usepackage{xcolor}         
\usepackage{makecell}
\usepackage{multirow}
\usepackage[table,xcdraw]{xcolor}
\usepackage[normalem]{ulem}
\useunder{\uline}{\ul}{}

\usepackage{microtype}
\usepackage{graphicx,enumitem}

\usepackage{array}
\usepackage[caption=false,font=normalsize,labelfont=sf,textfont=sf]{subfig}

\usepackage{graphicx,enumitem}
\usepackage{multirow}
\usepackage{amssymb}
\usepackage{mathtools}
\usepackage{amsthm}

\usepackage{multirow,enumitem,enumerate}
\usepackage{amssymb} 
\usepackage{tikz}
\usepackage{graphicx}
\usepackage{microtype}
\usepackage{multirow}
\usepackage{amssymb}
\usepackage{mathtools}
\usepackage{amsmath}
\usepackage{amsthm}
\usepackage{arydshln}
\usepackage{pifont}
\definecolor{markChanges}{rgb}{1,0,1}

\usepackage[most]{tcolorbox}  

\usepackage{hyperref}

\usepackage{algorithm}
\usepackage{algpseudocode}

\usepackage{breakurl}
\usepackage{textcomp}
\usepackage{stfloats}
\usepackage{url}
\usepackage{verbatim}
\usepackage{graphicx}
\usepackage{dsfont}

\usepackage{wrapfig}
\usepackage[dvipsnames]{xcolor}

\usepackage{dsfont}

\renewcommand{\arraystretch}{1.3}
\usepackage[capitalize,noabbrev]{cleveref}

\theoremstyle{plain}
\newtheorem{theorem}{Theorem}[section]
\newtheorem{proposition}[theorem]{Proposition}

\theoremstyle{definition}

\theoremstyle{remark}

\newcommand{\method}{FESTA}

\newcommand{\methodexpbold}{\textbf{F}unctionally \textbf{E}quivalent \textbf{S}ampling for \textbf{T}rust \textbf{A}ssessment}
\newcommand{\blue}[1]{\textcolor{black}{#1}}

\title{FESTA: Functionally Equivalent Sampling for Trust Assessment of Multimodal LLMs}

\author{Debarpan Bhattacharya\textsuperscript{*} \\
  Indian Institute of Science \\
  Bangalore, India \\
  \texttt{debarpanb@iisc.ac.in} \\\And
  Apoorva Kulkarni\textsuperscript{*} \\
  University of Maryland \\
  College Park, USA \\
  \texttt{apoorvak@umd.edu} \\\And
  Sriram Ganapathy \\
  Indian Institute of Science \\
  Bangalore, India \\
  \texttt{sriramg@iisc.ac.in}
\\
  }

\begin{document}
\maketitle
\begin{abstract}
The accurate trust assessment of multimodal large language models (MLLMs) generated predictions, which can enable selective prediction and improve user confidence, 
  is challenging due to the diverse multi-modal input paradigms.
  We propose \methodexpbold{}  (\method), a multimodal input sampling technique for MLLMs, that generates an uncertainty measure based on the
    equivalent and complementary input samplings. 
  The proposed task-preserving sampling approach for uncertainty quantification expands the input space to probe the consistency (through equivalent samples) and sensitivity (through complementary samples) of the model.
  FESTA uses only input-output access of the model (black-box), and does not require ground truth (unsupervised).
  The experiments are conducted with various off-the-shelf multi-modal LLMs, on both visual and audio reasoning tasks. The proposed \method{} uncertainty estimate achieves significant improvement
  ($33.3\%$ relative improvement for vision-LLMs and $29.6\%$ relative improvement for audio-LLMs) in selective prediction performance, based on area-under-receiver-operating-characteristic curve
(AUROC) metric in detecting mispredictions. The code implementation is open-sourced~\footnote{\noindent\href{https://github.com/iiscleap/multimodal-llm-uncertainty-estimation}{https://github.com/iiscleap/mllm-uncertainty-estimation}\\\text{\,\,\,\,\,\,\,\,*} Equal contribution.\\\text{\,\,\,\,\,\,\,\, \textit{Accepted in the Findings of EMNLP, 2025}}.}.
\end{abstract}

\section{Introduction}

Large language models (LLMs) have achieved remarkable performance across a wide array of natural language processing tasks \citep{brown2020gpt3, touvron2023llama, openai2023gpt4}. \blue{However, their performance in integration with other modalities- called, multimodal LLMs (MLLMs), is often inferior compared to text-only LLMs~\cite{li2024multi, fu2024blink}.}
\subsection{Selective prediction}
\blue{The works on selective prediction (SP) propose to prevent a model from answering in highly uncertain settings, thereby abstaining from potentially  wrong predictions (abstention)~\cite{wen2025know, madhusudhan2025llms}. Selective prediction is highly desirable for building safe AI deployments~\cite{amodei2016concrete, hendrycks2021unsolved} in: \\
\textbf{(a)} Tasks where a model offers low accuracy.\\
\textbf{(b)} Safety-critical scenarios like finance, medicine and autonomous driving, where incorrect predictions are very expensive.}\\
An efficient SP algorithm offers \textit{low selective risk} (high accuracy in the subset of questions it chooses to answer) in deployment.
\subsection{Selective prediction for multimodal LLMs}
\blue{One of the most common approaches to SP for LLMs is based on quantifying the variance in the output in the form of an entropy measure~\cite{kuhn2023semantic, farquhar2024detecting, ling2024uncertainty}. The model predictions with high uncertainty (and entropy) are more prone to errors, and hence are abstained from prediction. However, in some cases,  inaccurate predictions by MLLMs may arise from their insensitivity to the input (referred to as mode collapse). Such mis-predictions have low predictive entropy, making the entropy an unreliable measure of prediction accuracy. Other works also find LLMs to generate low-uncertainty hallucinations~\cite{simhi2025trust}. In many of these works, the log-probability has been the most common means to obtain output confidence in traditional neural networks. However, for MLLMs,\\
\textbf{(a)} Log-probability might not be accessible for closed-source models.\\
\textbf{(b)} The calibration of log-probability may be upset during the instruction tuning phase~\cite{tian2023just}.\\
\textbf{(c)} Calibration of log-probability deteriorates as model performance degrades~\cite{guo2017calibration, wang2023llmjudge, kurbis2024uncertainty}.\\
Hence, the need for uncertainty and confidence estimation for MLLMs, in black-box settings and low-performing tasks, is paramount.}
\blue{Recently published works~\cite{fu2024blink, bhattacharya25b_interspeech} report that MLLMs are astonishingly poor in simple multimodal reasoning tasks. Although this makes multimodal reasoning a suitable application for evaluating abstention algorithms, computation of predictive uncertainty of MLLMs remains challenging~\citep{hendrycks2022scaling}. 
}
\begin{figure}[t!]
    \centering
    \includegraphics[width=0.9\linewidth, height=0.6\textwidth]{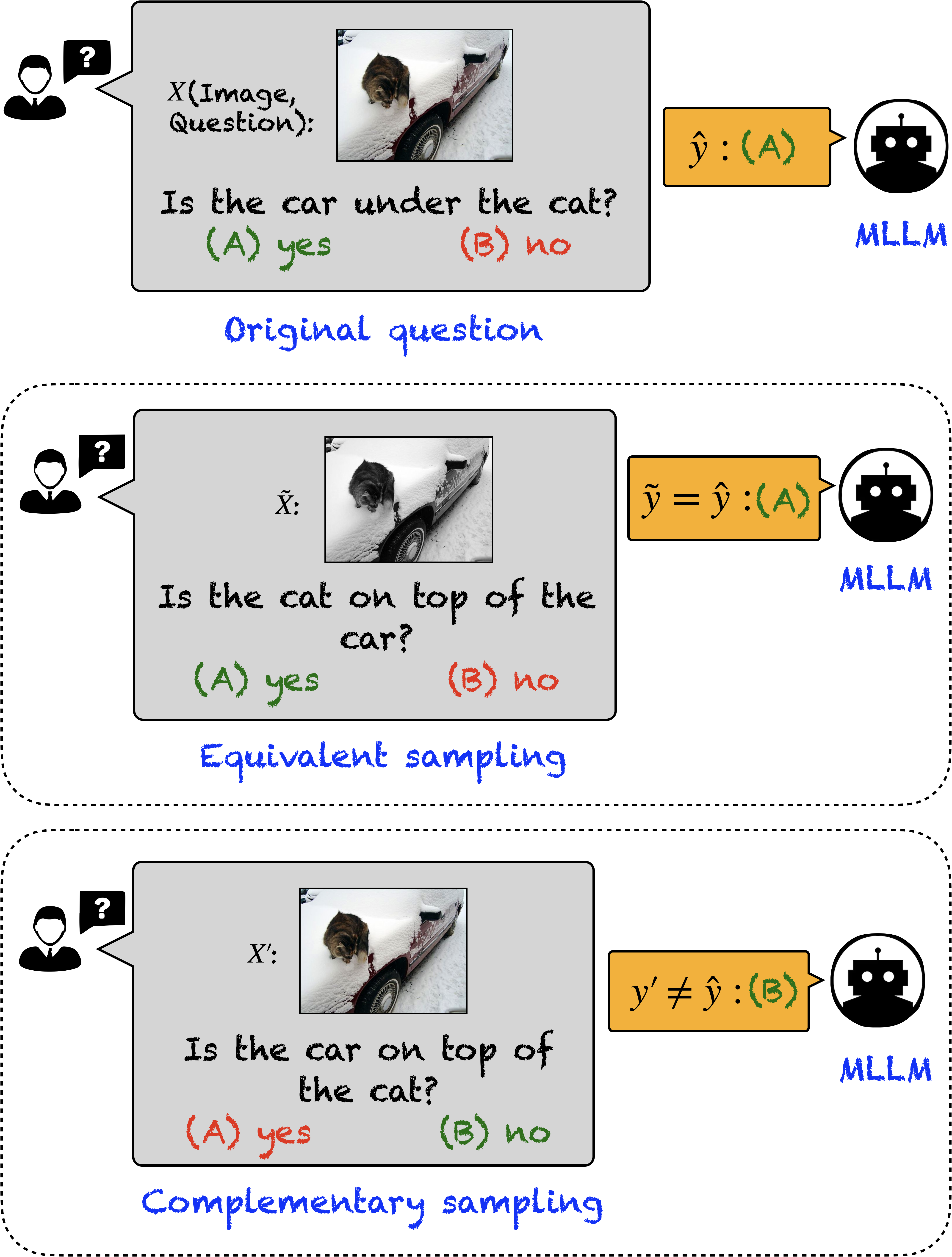}
    \caption{\blue{An example of multi-modal reasoning input (top-panel).   An equivalent  sample (middle panel) with same gray-scale image and rephrased prompt question is expected to keep the MLLM prediction unchanged, whereas a complementary input sample (bottom panel) is expected to alter the prediction. The proposed \method{} uses equivalent and complementary samples to generate the uncertainty measure.}
    }
    \vspace{-0.2in}
    \label{fig:summary}
\end{figure}
\begin{figure*}[t]
    \centering
    \includegraphics[width=0.95\linewidth, height=0.37\textwidth]{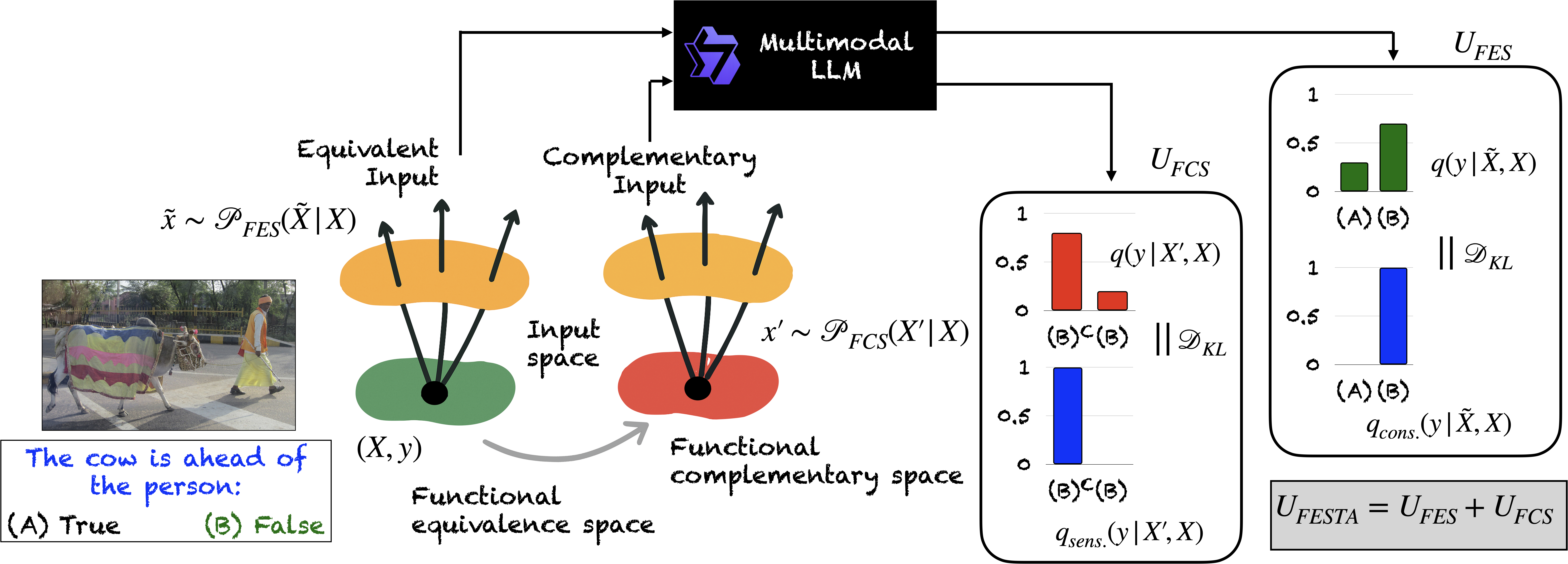}
    \caption{Schematic illustration of the proposed FESTA uncertainty quantification approach. Given a multimodal MCQ  input, we  generate functional equivalent samples (FES) and functional complementary samples (FCS). We compute divergence of model predictive uncertainty from an ideally consistent model (for FES) and an ideally sensitive model (for FCS), and then combine these measures to generate the \method{} uncertainty score.
    }
    \vspace{-0.2in}
    \label{fig:schematic}
\end{figure*}
\subsection{Contribution}
Based on the above discussion, and with the advent of multimodal LLMs such as large vision language models (LVLM)~\cite{liu2023visual, agrawal2024pixtral, bai2025qwen2} and large audio language models (LALM)~\cite{chu2024qwen2,tang2023salmonn}, we make the following observations:\\
\textbf{(O1)} The log-likelihood based quantification of model confidence is inapt for MLLMs because of its unavailability in black-box settings and its upset calibration during instruction tuning.\\
\textbf{(O2)} The existing output entropy-based uncertainty measures for black-box models fail in the case of low-entropy hallucinations.\\
\textbf{(O3)} Abstention performance of such entropy-based prior works also degrades in low accuracy tasks like reasoning.\\

\blue{In this paper, we propose \textit{functional equivalence sampling (FES)} and \textit{functional complementary sampling (FCS)} for abstaining from predictions in multimodal LLMs. FES samples are equivalent to the original input for the given task, and hence, ideally should not change the model prediction from the original input. FCS samples are complementary samples for the same task that are expected to alter the model prediction.
The examples are illustrated in Figure~\ref{fig:summary}. The MLLM predictions for FES and FCS samples are used to compute the FESTA uncertainty, which acts as an indicator of mis-prediction.}
We restrict this work to multiple-choice reasoning involving audio/visual prompts.
The key contributions are:
\begin{itemize}[leftmargin=*, itemsep=0pt, topsep=0pt]
    \item We propose equivalent (FES) and complementary (FCS) input samplings to identify the consistency and sensitivity of model outputs. They contribute to obtain the FESTA uncertainty score that is used to abstain from MLLM mispredictions in unsupervised and black-box settings.
    \item We quantify \method{} score as a KL-divergence from an ideally consistent and sensitive model, that improves over the  entropy-based measures.
    \item FESTA successfully mitigates low uncertainty hallucinations, where the other baselines fail.
    \item Extensive benchmark comparisons with various other prior works on audio/visual reasoning tasks and multiple open-source MLLMs.
\end{itemize}
Figure~\ref{fig:schematic} illustrates the working of \method{}.
\section{Problem statement}
While \method{} is applicable to general multi-modal outputs, we make a few assumptions and restrictions. We restrict this paper to the purely textual output case, in a multi-choice question-answering (MCQA) setting.
Further, it is also assumed that the MLLMs under study are instruction-tuned well enough to follow instructions to generate single-token MCQA outputs. \blue{Now, we define the problem statement, followed by formulating the FESTA algorithm. A summary of the notations used is present in Appendix Table~\ref{tab:notations-table}
}.\\
Let $\X = [\mathbf{\X_O}, \mathbf{\X_T}]\in  \Xc$ denote a multi-modal input instance (e.g., an audio/image + textual prompt) to an MLLM with ground-truth response $\y_\texttt{target}$, and let $\mathcal{S}$ denote the finite set of possible MCQA outputs.
The MLLM defines a predictive distribution over a fixed set of outputs:
$q(\y|\X) \in \Delta^{|\mathcal{S}|}$.
The model's prediction is (greedy sampling):
\[
\hat{\textbf{y}} := \arg\max_{\textbf{y} \in \mathcal{S}} q(\textbf{y}|\X).
\]
\begin{tcolorbox}[colback=gray!10!white, colframe=gray!70!black, boxrule=0.4pt, arc=2pt]
\textbf{Problem statement:} Given multi-modal input $\X$ and model output $\hat{\mathbf{y}}$ generated from inherent predictive distribution $q(\y|\X)$, estimate the predictive uncertainty of $\hat{\mathbf{y}}$ without access to $q(\y|\X)$ in a black-box setting.
\end{tcolorbox}
We resort to the following directions to develop the uncertainty estimator,
\begin{itemize}[leftmargin=*, itemsep=0pt, topsep=0pt]
    \item Sampling FES to estimate the epistemic uncertainty (consistency).
    \item Sampling FCS to measure counterfactual uncertainty (sensitivity).
    \item Combining the two measures to compute the \method{} uncertainty score.
\end{itemize}

\section{\method{} Uncertainty Estimator}
\label{sec:method}
\blue{In this section, we describe the sampling processes, and the computation of FESTA uncertainty score.}
\subsection{Functional equivalent samples (FES)}
\label{subsec:methods-fes}
 Given an input–output pair $(\X, \y_\texttt{target})$, let $T(\cdot)$ denote the  task that the model must solve to generate the correct response, and let $\textbf{M}_{\texttt{ideal}}$ denote the hypothetical model which has the ideal behavior.
\begin{tcolorbox}[colback=gray!10!white, colframe=gray!70!black, boxrule=0.4pt, arc=2pt]
\textbf{Def: }
A transformation $\tilde{\mathbf{X}} = \mathcal{E}(\mathbf{X})$ is  said as a \emph{functionally equivalent sample} of $\mathbf{X}$ if:
\[
T(\tilde{\mathbf{X}}) = T(\mathbf{X}) \, \text{and} \,M_{\texttt{ideal}}(\tilde{\mathbf{X}}) = M_{\texttt{ideal}}(\mathbf{X})
\]
\end{tcolorbox}
We define a distribution $\mathcal{P}_{\mathrm{FES}}(\tilde{\mathbf{X}}|\mathbf{X})$ over all possible equivalent transformations ($\{\mathcal{E}_i(\mathbf{X})\}_i$) and denote the sampling process as:
\[
\tilde{\mathbf{X}} \sim \mathcal{P}_{\mathrm{FES}}(\tilde{\mathbf{X}}|\mathbf{X}) \, 
\text{or} \, \tilde{\mathbf{X}}\sim_{\mathcal{E}} \mathbf{X}
\]
\blue{For example, in Figure~\ref{fig:summary}, the task a model must perform to answer $\mathbf{X}$ is spatial reasoning involving the objects - car and the cat. An equivalent sample (FES), $\tilde{\mathbf{X}}$  is a transformation of the input $\mathbf{X}$ (image, question) so that the spatial relationship of the cat and the car are equivalent to the one in the original image, or the rephrased prompt preserves the query part of the original question semantically. Thus, the FES   transformed image and question remain consistent with the original input. An example FES sample ($\tilde{\mathbf{X}}$) in Figure~\ref{fig:summary} involves the original image and the paraphrased question. In an ideal setting, the output of the MLLM  should be unaltered. In an analogous fashion, the audio version of FES, for different temporal reasoning tasks, are shown in Figure~\ref{fig:fes_fcs_examples}.in Appendix.}
Also, the FES samples hold formal equivalence among them (Appendix~\ref{sec:equivalence-proof-fes}).
\subsection{Functional complementary samples (FCS)}
\label{subsec:methods-fcs}
\begin{tcolorbox}[colback=gray!10!white, colframe=gray!70!black, boxrule=0.4pt, arc=2pt]
\textbf{Def:}
A transformation ${\mathbf{X}'} = \mathcal{C}(\mathbf{X})$ is a \emph{functionally complementary sample} of $\mathbf{X}$ if: 
\[
T({\mathbf{X}'}) = T(\mathbf{X}) \, \text{and} \,M_{\texttt{ideal}}({\mathbf{X}'}) \neq M_{\texttt{ideal}}(\mathbf{X})
\]
\end{tcolorbox}
So, an FCS sample is task-equivalent but functionally divergent. We define a distribution $\mathcal{P}_{\mathrm{FCS}}({\mathbf{X}'}|\mathbf{X})$ over all complementary transformations ($\{\mathcal{C}_i(\mathbf{X})\}_i$).
We sample $\mathbf{X}'$ as,
\[
{\mathbf{X}'} \sim \mathcal{P}_{\mathrm{FCS}}({\mathbf{X}'}|\mathbf{X}) \, \text{or} \, \mathbf{X}'\sim_{\mathcal{C}} \mathbf{X}
\]
\blue{For example, in Figure~\ref{fig:summary}, the complementary sample ${\mathbf{X}'}$ is a transformation of the image that doesn't alter the spatial relationship between the cat and the car, and the question is semantically reversed. The combination ensures that a robust model  will modify its response with respect to the original prediction. Alternatively, the image can be transformed complementarily as well. In an analogous fashion, the FCS perturbations of the audio examples are shown in Section~\ref{sec:fes-fcs-details}. Note that, negation words like ``not'' are avoided while generating complementary questions to avoid  model sensitivity to negation words~\cite{truong2023language}.}
Although a complementary sample is not equivalent to the original input,  it can be shown that all complementary samples $\mathbf{X}'$ have formal equivalence among them (Appendix~\ref{sec:equivalence-proof-fcs}). \blue{Further details on the generation of FES and FCS samples are discussed in Sections~\ref{subsec:fes-generation} and~\ref{subsec:fcs-generation}}.
\subsection{Ideal behavior under FES and FCS}
 Before uncertainty quantification, we first introduce the notion of a consistent and sensitive model, $M_{\texttt{cons.}}$ and $M_{\texttt{sens.}}$, respectively.  Their definitions do not rely on the ground truth $\y_\texttt{target}$, making the approach fully unsupervised. They have a subset of the properties that $M_{\texttt{ideal}}$ has.
\begin{itemize}[leftmargin=*, itemsep=0pt, topsep=0pt]
    \item \textcolor{black}{\textbf{Consistency under FES}: The model $M_{\texttt{cons.}}$ generates predictions that remain consistent and unaltered with respect to the original prediction, i.e., ($\y = \hat{\y}$), under FES.}
    \item \textcolor{black}{\textbf{Sensitivity to FCS}: The  model $M_{\texttt{sens.}}$ is sensitive to counterfactual negations and its predictions for FCS are complementary to the original predictions, i.e., ($\y\neq \hat{\y}$). 
    }
\end{itemize}
We later define the consistency entropy ($U_{\texttt{FES}}$) and sensitivity entropy ($U_{\texttt{FCS}}$), based on the deviation of a given model $M$ from $M_{\texttt{cons.}}$ and $M_{\texttt{sens.}}$, respectively.
Note that, an ideal model $M_{\texttt{ideal}}$ satisfies properties of both $M_{\texttt{cons.}}$ and $M_{\texttt{sens.}}$ ($M_{\texttt{ideal}} \subset M_{\texttt{cons.}}$, $M_{\texttt{ideal}} \subset M_{\texttt{sens.}}$). $M_{\texttt{ideal}}$ is both consistent and sensitive, while the converse is false. 
\subsection{Uncertainty Estimation from FES}
For an input $\X = \x$, denote its FES samples as: $\{ \tilde{\x}_1, \tilde{\x}_2, \dots, \tilde{\x}_{K_1} \}$.
Let $q(\y|\tilde{\x}_k)$ be the predictive distribution of model $M$ for $\tilde{\x}_k$ and $\y$ is the output random variable sampled using the stochastic decoding of the MLLMs.
The uncertainty measure of the model $M$ is measured as the deviation from the grounded model $M_{\texttt{cons.}}$.\\
\begin{itemize}[leftmargin=*, itemsep=0pt, topsep=0pt]
    \item \textbf{Predictive distribution of $M_{\texttt{cons.}}$}: The definition of $M_{\text{cons.}}$ implies that its predictive distribution is a Kronecker delta function with respect to equivalent sampling. The entire weight of the delta is concentrated at the greedy output $\hat{\y}$.
    \begin{align*}
    q_{\texttt{cons.}}(\y \mid \tilde{\x}_k, \x) &= \delta_{\y\hat{\y}} \quad \forall\, \tilde{\x}_k \sim_{\mathcal{E}} \mathcal{P}_{\mathrm{FES}},\; \y \in \mathcal{S} 
    \end{align*}
    \item \textbf{Predictive distribution of $M$}: This is given as:
    \begin{align*}
q_{FES}(\y \mid \x) &= \mathbb{E}_{\tilde{\x}_k \sim \mathcal{P}_{\mathrm{FES}}} \, q(\y \mid \tilde{\x}_k, \x) \, , \y \in \mathcal{S}
\end{align*}
\end{itemize}
Now, we pose the uncertainty of model $M$ as the deviation of its predictive distribution from that of $M_{\texttt{cons.}}$, i.e.,
\[
U_{FES}(M|\x) = D_{KL}(q_{\texttt{cons.}}(\y \mid \x) || q_{FES}(\y \mid \x))
\]
\begin{proposition}
The $U_{FES}(M|\x)$ simplifies to:
\[
U_{\text{FES}}(M \mid \mathbf{x}) := -\log q_{\text{FES}}(\y = \hat{\y} \mid \mathbf{x}), \y\in \mathcal{S}
\]
\end{proposition}
\begin{proof}
The proof is given in the Appendix~\ref{sec:fes-uncertainty-closed-form}.
\end{proof}
\blue{It is interesting to note that, resorting to $M_{\text{cons.}}$ makes the quantification self-supervised. While it is solely dependent on model predictions without access to ground truth (unsupervised), it is based on the deviation from a hypothetically ideal model.}
\subsection{Uncertainty Estimation from FCS}
The uncertainty from complementary samples: $\{\x'_1, \x'_2, \dots, \x'_{K_2} \}$, drawn from $\x'_k \sim \mathcal{P}_{\mathrm{FCS}}$, is,\\
\begin{itemize}[leftmargin=*, itemsep=0pt, topsep=0pt]
    \item \textbf{Predictive distribution of $M_{\texttt{sens.}}$}: Based on the definition of FCS, they alter the predictions - $M_{\texttt{sens.}}(\x') \neq M_{\texttt{sens.}}(\x)$. Clearly, under the original support $\mathcal{S}$, the predictive distribution $q(\y \mid \x_k', \x)$ can be any distribution. But, we restrict the support to have only two members as $\mathcal{S}' = \{\hat{\y}, \hat{\y}^c\}$ where $\hat{\y}^c = \{\y: \y\in \mathcal{S}, \y\neq \hat{\y}\}$. Now, the predictive distribution becomes a Kronecker delta with entire mass on $\hat{\y}^c$.
    \begin{align*}
    q_{\texttt{sens.}}(\y \mid \x_k', \x) &= \delta_{\y\hat{\y}^c} \quad \forall\, \x_k' \sim_{\mathcal{C}} \mathcal{\X},\; \y \in \mathcal{S}' 
    \end{align*}
    \item \textbf{Predictive distribution of $M$}: In this case,
    \begin{align*}
q_{FCS}(\y \mid \x) &= \mathbb{E}_{\x_k' \sim \mathcal{P}_{\mathrm{FCS}}} \, q(\y \mid \x_k', \x) \, , \y\in \mathcal{S}'
\end{align*}
\end{itemize}
Now, the uncertainty of model $M$ from FCS is the deviation from $M_{\texttt{sens.}}$ in predictive distribution:
\[
U_{FCS}(M|\x) = D_{KL}(q_{\texttt{sens.}}(\y \mid \x) || q_{FCS}(\y \mid \x))
\]
\begin{proposition}
The $U_{FCS}(M|\x)$ simplifies to:
\[
U_{\text{FCS}}(M \mid \mathbf{x}) := -\log \left(\sum_{\y} q_{\text{FCS}}(\y \neq \hat{\y} \mid \mathbf{x})\right)
\]
\end{proposition}

\begin{proof}
The proof is similar to Appendix~\ref{sec:fes-uncertainty-closed-form}.
\end{proof}
FES samples do not provide the right framework for quantifying the low-uncertainty hallucinations like insensitivity to the prompt. However, $U_{FCS}$ helps to abstain from such low-uncertainty mis-predictions as such models show no sensitivity to the complementary samples. 
We formally show this for a single attention block in Appendix~\ref{sec:fcs-low-uncertainty-hallucinations}.
\subsection{FESTA uncertainty estimate}
The FESTA uncertainty estimate combines two axes of uncertainty quantification - $U_{FES}$, which measures the model consistency for equivalent samples, and $U_{FCS}$, which measures the model sensitivity for complementary samples.
The combination helps to abstain both high uncertainty and low uncertainty mis-predictions, as evident in Section~\ref{sec:results}. 
\method{} is detailed in Figure~\ref{fig:schematic} and Algorithm~\ref{alg:festa}.
\begin{algorithm}[h!]
\caption{FESTA Uncertainty Estimator}
\label{alg:festa}
\begin{algorithmic}[1]
\Require Input $\X$, original prediction $\hat{\y} = \arg\max_{\y} q(\y|\X)$, no. of samples $K = K_1 + K_2$.
\State \textbf{FES Sampling:}
\State Generate $K_1$ FES samples: $\{\tilde{\x}_k\}_{k=1}^{K_1} \sim P_{\text{FES}}(\tilde{\x}|\X)$ by sampling text and non-text modalities ($\X_T, \X_O$) for ($K_{11}, K_{12}$) times and using all combinations ($K_1=K_{11}\times K_{12}$).
\State Compute $q_{\text{FES}}(\y|\X)$ using $\{{\tilde{\x}_k\}_{k=1}^{K_1}}$
\State Compute FES uncertainty for prediction $\hat{\y}$:
\begin{equation*}
\vspace{-0.05in}
U_{\text{FES}} = -\log q_{\text{FES}}(\y=\hat{\y}|\X)
\vspace{-0.03in}
\end{equation*}

\State \textbf{FCS Sampling:}
\State Generate $K_2$ FCS samples: $\{\x'_k\}_{k=1}^{K_2} \sim P_{\text{FCS}}(\x'|\X)$ by complementary sampling either of the text or non-text modalities (e.g. $\X_T$) for $K_{21}$ times and equivalent sampling the other (e.g. $\X_O$) for $K_{22}$ times, and use all combinations ($K_2=K_{21}\times K_{22}$).
\State Compute $q_{\text{FCS}}(\y|\X)$ using $\{\x'_k\}_{k=1}^{K_2}$.
\State Compute FCS uncertainty for prediction $\hat{\y}$:
\begin{equation*}
U_{\text{FCS}} = -\log\left(\sum_{\y \ne \hat{\y}} q_{\text{FCS}}(\y|\X)\right)
\end{equation*}
\State \textbf{FESTA:} $U_{\text{FESTA}} =$ $U_{\text{FES}}$+ $U_{\text{FCS}}$
\end{algorithmic}
\end{algorithm}

\section{Related Prior Work}
\textbf{Uncertainty estimation for LLMs:} \blue{Uncertainty estimation for Large Language Models (LLMs) is understudied, yet an important area as the uncertainty measure seldom correlates with their task accuracy~\cite{ye2024benchmarking}.}\\
\noindent There can be black-box or white-box approaches to LLM uncertainty estimation~\cite{shorinwa2025survey, wen2025know}. White-box methods leverage internal model access to quantify uncertainty~\cite{chen2024inside}. However, such approaches are restrictive for closed-source models~\cite{lin2024generating}. Black-box methods are suitable for closed-source models as they assume only query access~\cite{gao-etal-2024-spuq, kadavath2022language}, such as input ensembling, augmenting, or rephrasing to estimate predictive entropy~\cite{hou2024decomposing, yang2024just, jiang2023calibrating}. Other approaches for uncertainty estimation for LLMs use conformal prediction~\cite{yadkori2024mitigating, wang-etal-2025-sconu, confot25}. Critical gaps remain in uncertainty estimation for LLMs regarding robustness and efficient sampling strategies~\cite{shorinwa2025survey, abbasi-yadkori2024to, yona2024can}.\\
\noindent \textbf{Uncertainty estimation for MLLMs:} \blue{Although the multimodal LLMs suffer more from erroneous predictions and hallucinations, effective abstention algorithms for them are scarce. The amplified hallucinations for MLLMs can be attributed to the ineffective grounding~\cite{favero2024multi} and fusion~\cite{kang2025see}. A handful of other studies, that reduce MLLM hallucination, model the problem as a causal graph~\cite{li2025treblecounterfactualvlmscausal}, mutual-information decoding with direct preference optimization~\cite{favero2024multi}, and visual attention redistribution~\cite{kang2025see}. However, they need training or access to internal model parameters/log-probabilities. Other works employ uncertainty-aware agentic framework~\cite{zhi2025seeingreasoningconfidencesupercharging}, semantic perturbation~\cite{zhao2025objectlevelverbalizedconfidencecalibration, khan2024consistency} or self-assessment to compute uncertainty~\cite{chen2025unveiling}. In contrast, our work is a fully unsupervised, black-box model-based and post-hoc uncertainty estimation method.}\\
\noindent \blue{The predictive uncertainty estimation can be based on open-ended or multiple-choice (MCQA) predictions. Because of its limited output, simplicity, and faster inference~\cite{zhang2024multiple}, the MCQA framework is widely used in accuracy and uncertainty evaluation benchmarks for  LLMs~\cite{hendrycks2021measuring, wang2024ubench, yang-etal-2025-maqa} and MLLMs~\cite{liu2024mmbench, fu2024mmecomprehensiveevaluationbenchmark, sakshi2025mmau}. However, the free-text applications require uncertainty estimation for open-ended predictions, like semantic entropy~\cite{kuhn2023semantic,farquhar2024detecting} and iterative prompting~\cite{abbasi-yadkori2024to}.}
\section{Experimental Setup}
\subsection{Tasks and Datasets}
We use $3$ datasets on positional reasoning - spatial reasoning for the vision-LLMs and temporal reasoning for audio-LLMs, in this study.\\
\textbf{BLINK}: The BLINK dataset~\cite{fu2024blink} points to the limitations of modern vision-LLMs with binary questions, where relative positions between different objects in an image are queried. The spatial reasoning samples from the validation split ($143$ samples) are chosen for evaluation.\\
\textbf{VSR}: The Visual spatial reasoning (VSR) dataset~\cite{liu2023vsr} also has MCQ questions with two choices. It is entirely focused on diverse spatial reasoning samples. The validation partition with $100$ randomly samples are used for evaluation.\\
\textbf{TREA}: The Temporal Reasoning Evaluation of Audio (TREA) dataset~\cite{bhattacharya25b_interspeech} is a comprehensive audio-temporal reasoning dataset on which audio-LLMs perform poorly~\cite{kuan2025can}. It has MCQ with four answer choices. It further divides the temporal reasoning task into $3$ categories - ordering, duration, and event counting.  We use a subset of $300$ samples ($100$ per task).
\subsection{Multimodal LLMs}
\textbf{Visual spatial reasoning}: We use large vision language models (LVLM)- \texttt{Gemma-3}~\cite{team2025gemma}, \texttt{LLaVA-1.6}~\cite{liu2023visual}, \texttt{Qwen-2.5VL}~\cite{yang2025qwen3}, \texttt{Phi4}~\cite{abdin2024phi} and \texttt{Pixtral}~\cite{agrawal2024pixtral} for the evaluation. All of these models show significantly lower performance (Table~\ref{tab:vlm-results-spatial-reasoning}) compared to humans performance of $>95\%$~\cite{fu2024blink}. Further details and the model cards are in Appendix~\ref{appendix:model-cards}.\\
\textbf{Audio temporal reasoning}: We evaluate two open-source audio-LLMs \texttt{Qwen2-audio}~\cite{chu2024qwen2} and \texttt{SALMONN}~\cite{tang2023salmonn}. 
We observe their temporal reasoning performance to be poor for the tasks (Table~\ref{tab:alm-results-temporal-reasoning}). 
We have also experimented with generating audio captions using the audio-LLMs and then passing the text captions to a text-only LLM \texttt{Qwen-2}~\cite{bai2025qwen2} as suggested in~\cite{bhattacharya25b_interspeech}, for improved accuracy (Table~\ref{tab:alm-results-temporal-reasoning}). More details are in Appendix~\ref{appendix:model-cards}.



\subsection{Comparison with Baseline Systems}
\label{sec:baselines}
\begin{itemize}[leftmargin=*, itemsep=0pt, topsep=0pt]
    \item \textbf{Output entropy (OE)}: The predictive distribution of the models is estimated using stochastic decoding, and the entropy is measured as~\cite{kuhn2023semantic}: $\mathcal{H}\left( q(\y|\x) \right) = - \sum_{\y}q(\y|\x)\log q(\y|\x)$.
    \item \textbf{Verbalized confidence (VC)}: LLMs are good estimators of their own confidence~\cite{tian2023just}, when  verbalized  through  prompting.
    \item \textbf{Input augmentations (IA)}: Based on the approach in~\cite{bahat2020classification} to obtain predictions using input  augmentations and computing entropy from the ensemble. Apart from image augmentations (IA-I), we performed text augmentations (IA-T), using paraphrasing. Finally, we report performance of combined augmentations (IA-IT).
    \item \textbf{Rephrase uncertainty (RU)}: This system~\cite{yang2024just} uses text rephrasing and measures the answer consistency.
    \item \textbf{Black-box uncertainty (BU)}: The work is based on using a combination of top-K prompting and random output sampling that yields the most stable performance~\cite{xiong2024can}. We use top-$4$ prompting with outputs sampling $5$ times.
    \item \textbf{Semantic uncertainty}: \blue{Another popular uncertainty estimation method is the concept of semantic entropy~\cite{kuhn2023semantic, farquhar2024detecting}. It is designed primarily to cater to the open-text answers. However, in the MCQA settings, semantic entropy boils down to Output entropy (OE) (Appendix~\ref{appendix:semantic-entropy-mcq})}.
\end{itemize}
\subsection{Evaluation metric}
\label{sec:eval-metrics}
An uncertainty measure should correlate with the probability of a prediction being incorrect.
The performance of uncertainty methods is evaluated using Area-Under-Receiver-Operating-Curve (AUROC):
\[
\texttt{AUROC} = \texttt{AUC}\left( \frac{1}{U},  \mathds{1}_{\{\hat{\y}=\y_{\texttt{target}}\}}\right)
\]
where $U$ is the uncertainty, $\y _{\texttt{target}}$ denotes the ground-truth, $\hat{\y}$ denotes the model output. 
\subsection{Equivalent Samples (FES)}
\label{subsec:fes-generation}
\blue{Adding to the  description  provided in Section~\ref{subsec:methods-fes}, FES samples are the transformations of the original input that keep the task objective of the original question unchanged ideally.  Transformations used on images to create equivalent spatial reasoning samples are the RGB to grayscale transformation, adding Gaussian noise, and mild levels of blurring. Equivalent transformations on the text (the question) related to multiple para-phrasings. Different combinations of such image and text equivalent samples form the multimodal FES samples.
}

For the audio tasks, FES samples are transformations on the original input that maintain the task and functional equivalence. The transformations are done on both modalities: audio input and textual question. While some transformations, such as adding noise, varying the loudness, paraphrasing the textual question, etc, are generic, FES transformations can be task-specific too. For example, for the order task, where the goal is to answer which audio event occurred after the event $A$, changing the duration of event $A$ by trimming is an FES transformation. However, this is an invalid FES transformation with respect to the event duration task. 
 The complete description of equivalent sampling (FES) transforms used along with the sampling process, and a few examples are detailed in Appendix~\ref{sec:fes-fcs-details} and Figure~\ref{fig:fes_fcs_examples}.\\
\subsection{Complementary samples (FCS)}
\label{subsec:fcs-generation}
\blue{Based on the discussion in Section~\ref{subsec:methods-fcs}, FCS samples are the transformations of the original input that keep the task objective of the original question same, but alter the direction of the spatial/temporal reasoning. In an ideal setting, FCS samples force the MLLM to alter its response with respect to the original input.}

For example, for audio tasks, a complementary transform can be the addition of a new audio event at the start or end of the input audio clip for the audio event counting task. For images, it can be flipping an image horizontally for a ``at the left of'' or ``at the right of'' type of questions. For the text, it can be altering the spatial relationship from ``in front of'' to ``at the behind of''.
The FCS samples used in Table~\ref{tab:vlm-results-spatial-reasoning} use complemented text, as generalized complementary transform on images is difficult. \blue{However, as proof of the concept, we generate complementary image samples for the ``left of'' and ``right of'' questions, and note consistent improvements (Appendix~\ref{appendix:image-fcs-experiment})}.
The FCS transforms used are in Appendix~\ref{sec:fes-fcs-details} and Figure~\ref{fig:fes_fcs_examples}.
\section{Results}
\label{sec:results}
\begin{table*}[t!]
\centering
\resizebox{0.7\linewidth}{!}{
\def\arraystretch{1.4}
\begin{tabular}{lllcccccccc}
\toprule[1.2pt]
\multirow{2}{*}{\centering \textbf{Dataset}} & 
\multirow{2}{*}{\centering \textbf{Model}} & 
\multirow{1}{*}{\centering \textbf{Pred.}} & 
\multicolumn{7}{c}{\textbf{Baseline Results (AUROC)}} &{\textbf{Ours (AUROC)}}   \\
\cmidrule(lr){4-10} \cmidrule(lr){11-11}
& & \textbf{Acc.} & OE & VC & IA-I & IA-T & IA-IT & RU & BU & FESTA \\
\midrule[1.2pt]
\multirow{5}{*}{\textbf{BLINK}}
& \texttt{Gemma-3} & 0.80 & 0.53 & 0.61 & 0.55 & 0.68 & 0.63 & \underline{0.71} & 0.69 & \textbf{0.81 (\textcolor{ForestGreen}{14.1\%})} \\
& \texttt{LLaVA-1.6} & 0.71 & \underline{0.67} & 0.56 & 0.47 & 0.62 & 0.63 & 0.62& 0.51 & \textbf{0.77 (\textcolor{ForestGreen}{14.9\%})} \\
& \texttt{Qwen-2.5-VL} & 0.88 & \underline{0.86} & 0.65 & 0.77 & 0.78 & 0.80 & 0.77 & 0.60 & \textbf{0.93 (\textcolor{ForestGreen}{8.1\%})} \\
& \texttt{Phi-4} & 0.71 & 0.63 & 0.49 & 0.56 & 0.60 & \underline{0.64} & 0.58 & 0.38 & \textbf{0.87 (\textcolor{ForestGreen}{35.9\%})} \\
& \texttt{Pixtral} & 0.78 & 0.72 & 0.59 & \underline{0.75} & 0.70 & 0.73 & 0.75 & 0.58 & \textbf{0.90 (\textcolor{ForestGreen}{20.0\%})} \\
\cmidrule(lr){2-11}
& \textbf{Avg.} & 0.78 & 0.68 & 0.58 & 0.62 & 0.68 & 0.69 & \underline{0.69} & 0.55 & \textbf{0.86 (\textcolor{ForestGreen}{24.6\%})}\\
\midrule
\multirow{5}{*}{\textbf{VSR}}
& \texttt{Gemma-3} & 0.74 & 0.53 & 0.56 & 0.60 & 0.58 & 0.63 & 0.59 & \underline{0.66} & \textbf{0.88 (\textcolor{ForestGreen}{33.3\%})} \\
& \texttt{LLaVA-1.6} & 0.60 & 0.57 & 0.52 & 0.57 & 0.58 & \underline{0.66} & 0.56 & 0.53 & \textbf{0.74 (\textcolor{ForestGreen}{12.1\%})} \\
& \texttt{Qwen-2.5-VL} & 0.95 & 0.65 & 0.47 & 0.61 & \underline{0.65} & 0.61 & 0.63 & 0.59 & \textbf{0.92 (\textcolor{ForestGreen}{41.5\%})} \\
& \texttt{Phi-4} & 0.68 & 0.58 & 0.49 & 0.48 & 0.50 & \underline{0.60} & 0.50 & 0.56 & \textbf{0.94 (\textcolor{ForestGreen}{56.7\%})} \\
& \texttt{Pixtral} & 0.76 & 0.57 & 0.59 & 0.55 & 0.62 & 0.60 & \underline{0.68} & 0.61 & \textbf{0.91 (\textcolor{ForestGreen}{33.8\%})} \\
\cmidrule(lr){2-11}
& \textbf{Avg.} & 0.75 & 0.58 & 0.53 & 0.56 & 0.59 & \underline{0.62} & 0.59 & 0.59 & \textbf{0.88 (\textcolor{ForestGreen}{41.9\%})}\\
\bottomrule[1.2pt]
\end{tabular}
}
\caption{Results for vision-LLMs. The final column in green (\%) reports the relative improvement of \method{} over the best baseline result (underlined).}
\label{tab:vlm-results-spatial-reasoning}
\end{table*}
\begin{table*}[t!]
\centering
\resizebox{0.78\linewidth}{!}{
\def\arraystretch{1.3}
\begin{tabular}{lllcccccccc}
\toprule[1.2pt]
\multirow{2}{*}{\centering \textbf{Dataset}} & 
\multirow{2}{*}{\centering \textbf{Model}} & 
\multirow{1}{*}{\centering \textbf{Pred.}} & 
\multicolumn{7}{c}{\textbf{Baseline Results (AUROC)}} & 
\textbf{Ours (AUROC)} \\
\cmidrule(lr){4-10} \cmidrule(lr){11-11}
& & \textbf{Acc.} & OE & VC & IA-A & IA-T & IA-AT & RU & BU & FESTA \\
\midrule[1.2pt]
\multirow{3}{*}{\textbf{TREA-O}} 
& \texttt{Qwen2-Audio} & 0.51 & 0.66 & 0.67 & 0.65 & 0.68 & 0.63 & \underline{0.70} & 0.53 & \textbf{0.91 (\textcolor{ForestGreen}{30.0\%})} \\
& \texttt{SALMONN} & 0.31 & 0.54 & 0.54 & 0.55 & 0.55 & \underline{0.66} & 0.64 & 0.39 & \textbf{0.86 (\textcolor{ForestGreen}{30.3\%})} \\
& \texttt{SALMONN} desc. + LLM & 0.75 & 0.68 & 0.74 & 0.83 & 0.64 & \underline{0.85} & 0.68 & 0.65 & \textbf{0.90 (\textcolor{ForestGreen}{5.8\%})} \\
\cmidrule(lr){2-11}
& \textbf{Avg.} & 0.52 & 0.63 & 0.65 & 0.68 & 0.62 & \underline{0.71} & 0.67 & 0.52 & \textbf{0.89 (\textcolor{ForestGreen}{25.4\%})} \\
\midrule
\multirow{3}{*}{\textbf{TREA-D}} 
& \texttt{Qwen2-Audio} & 0.45 & 0.61 & \underline{0.75} & 0.55 & 0.62 & 0.59 & 0.66 & 0.56& \textbf{0.75 (\textcolor{gray}{0.0\%})} \\
& \texttt{SALMONN} & 0.35 & 0.54 & 0.48 & 0.50 & \underline{0.56} & 0.50 & 0.54 & 0.50 & \textbf{0.76 (\textcolor{ForestGreen}{35.7\%})} \\
& \texttt{SALMONN} desc. + LLM & 0.49 & 0.57 & 0.49 & 0.67 & 0.55 & \underline{0.69} & 0.56 & 0.54 & \textbf{0.80 (\textcolor{ForestGreen}{15.9\%})} \\
\cmidrule(lr){2-11}
& \textbf{Avg.} & 0.43 & 0.57 & 0.57 & 0.57 & 0.58 & \underline{0.59} & \underline{0.59} & 0.53 & \textbf{0.77 (\textcolor{ForestGreen}{30.5\%})} \\
\midrule
\multirow{3}{*}{\textbf{TREA-C}} 
& \texttt{Qwen2-Audio} & 0.21 & 0.49 & \underline{0.68} & 0.48 & 0.47 & 0.45 & 0.47 & 0.50 & \textbf{0.83 (\textcolor{ForestGreen}{22.1\%})} \\
& \texttt{SALMONN} & 0.20 & 0.34 & 0.40 & \underline{0.46} & 0.29 & 0.32 & 0.27 & 0.43 & \textbf{0.66 (\textcolor{ForestGreen}{43.5\%})} \\
& \texttt{SALMONN} desc. + LLM & 0.50 & 0.61 & \underline{0.66} & 0.54 & 0.54 & 0.45 & 0.55 & 0.62 & \textbf{0.81 (\textcolor{ForestGreen}{22.7\%})} \\
\cmidrule(lr){2-11}
& \textbf{Avg.} & 0.30 & 0.48 & \underline{0.58} & 0.49 & 0.43 & 0.41 & 0.43 & 0.52 & \textbf{0.77 (\textcolor{ForestGreen}{32.8\%})} \\
\bottomrule[1.2pt]
\end{tabular}
}
\caption{Results for audio-LLMs using \texttt{Qwen-2-Audio}~\cite{chu2024qwen2}, \texttt{SALMONN}~\cite{tang2023salmonn} and audio captions generated by \texttt{SALMONN} followed by text-only \texttt{Qwen-2} model (LLM). The final column in green (\%) reports the relative improvement of \method{} over the best baseline result (underlined).}
\label{tab:alm-results-temporal-reasoning}
\end{table*}


\subsection{FESTA uncertainty evaluation}
The performance of FESTA and other baselines on vision-LLMs and audio-LLMs are reported in Tables~\ref{tab:vlm-results-spatial-reasoning} and~\ref{tab:alm-results-temporal-reasoning}. We make the following observations:
\begin{itemize}[leftmargin=*, itemsep=0pt, topsep=0pt]
    \item The abstention AUROC of FESTA uncertainty score outperforms all the black-box baseline approaches significantly, achieving $24.6\%$ and $41.9\%$ relative improvements over the second-best approaches on BLINK and VSR datasets, respectively, averaged over different models. Also, for audio-LLMs, it achieves $25.4\%$, $30.5\%$, and $32.8\%$ relative improvements for order, duration and count tasks, respectively.
    \item Although the audio-LLM performances were poor for temporal reasoning (accuracy: order - $52$\%, duration - $43$\% and count - $30$\%), the AUROC achieved by \method{}  of $0.89$, $0.77$, and $0.77$ respectively, shows its effectiveness for the low-performing reasoning tasks.
    \item For vision-LLMs, \method{} shows effectiveness across low and high accuracy models. The largest improvements were observed for \texttt{Phi-4} (the model with the least size of $5.6$B).
    \item For both audio-LLMs and vision-LLMs, the baseline systems perform inconsistently. None of them consistently provide second-best performance across models and tasks.
\end{itemize}

\subsection{Ablations}
To further probe the improvements provided by \method{}, we conduct the following analyses:
\begin{itemize}
    \item FESTA uncertainty is a quantification of both equivalent and complementary input samplings. We separately analyze the  AUROCs  from equivalent (FES) and complementary (FCS) samples, as shown in Figures~\ref{fig:fes_fcs_barplot} and~\ref{fig:fes_fcs_barplot_vlm}. For the order and duration tasks in audio reasoning, AUROC for \texttt{Qwen2-Audio} is more influenced by the FCS, but \texttt{SALMONN} and \texttt{SALMONN} desc.+LLM models are more influenced by the FES. It shows the complementary nature of FES and FCS inputs under different scenarios. For the poorly performing event count task, FCS uncertainty contributes the most, showing its robustness in such low-accuracy scenarios.
    \begin{figure*}[t]
    \centering
    \includegraphics[width=1.0\linewidth, height=0.27\textwidth]{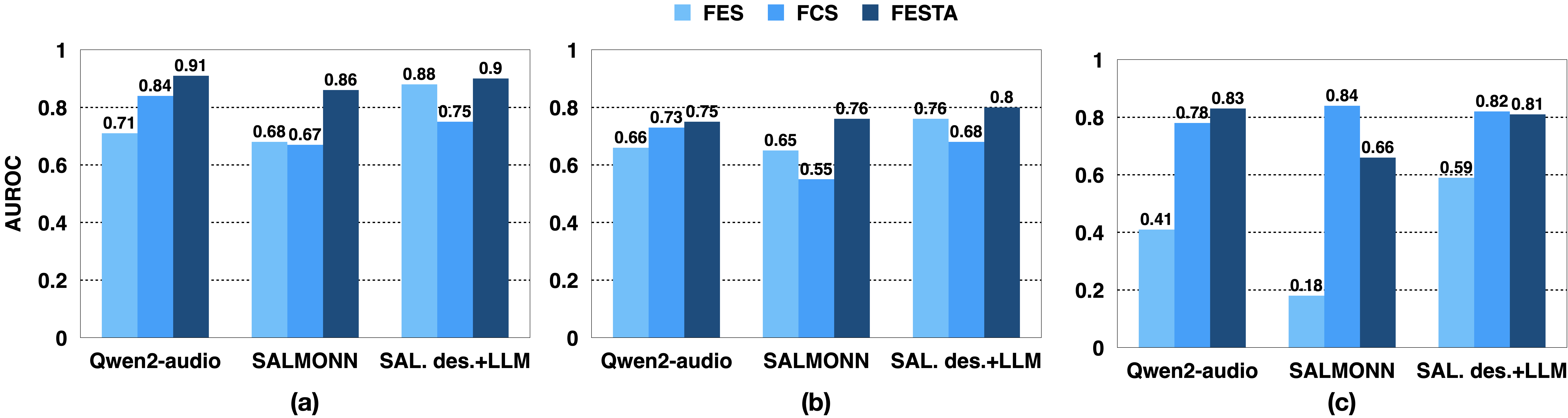}
    \caption{AUROC for uncertainty based on FES, FCS and \method{} on (a) TREA-O, (b) TREA-D, and (c) TREA-C.
    }
    \label{fig:fes_fcs_barplot}
\end{figure*}
\begin{figure}[t]
    \centering
    \includegraphics[width=1.0\linewidth, height=0.6\textwidth]{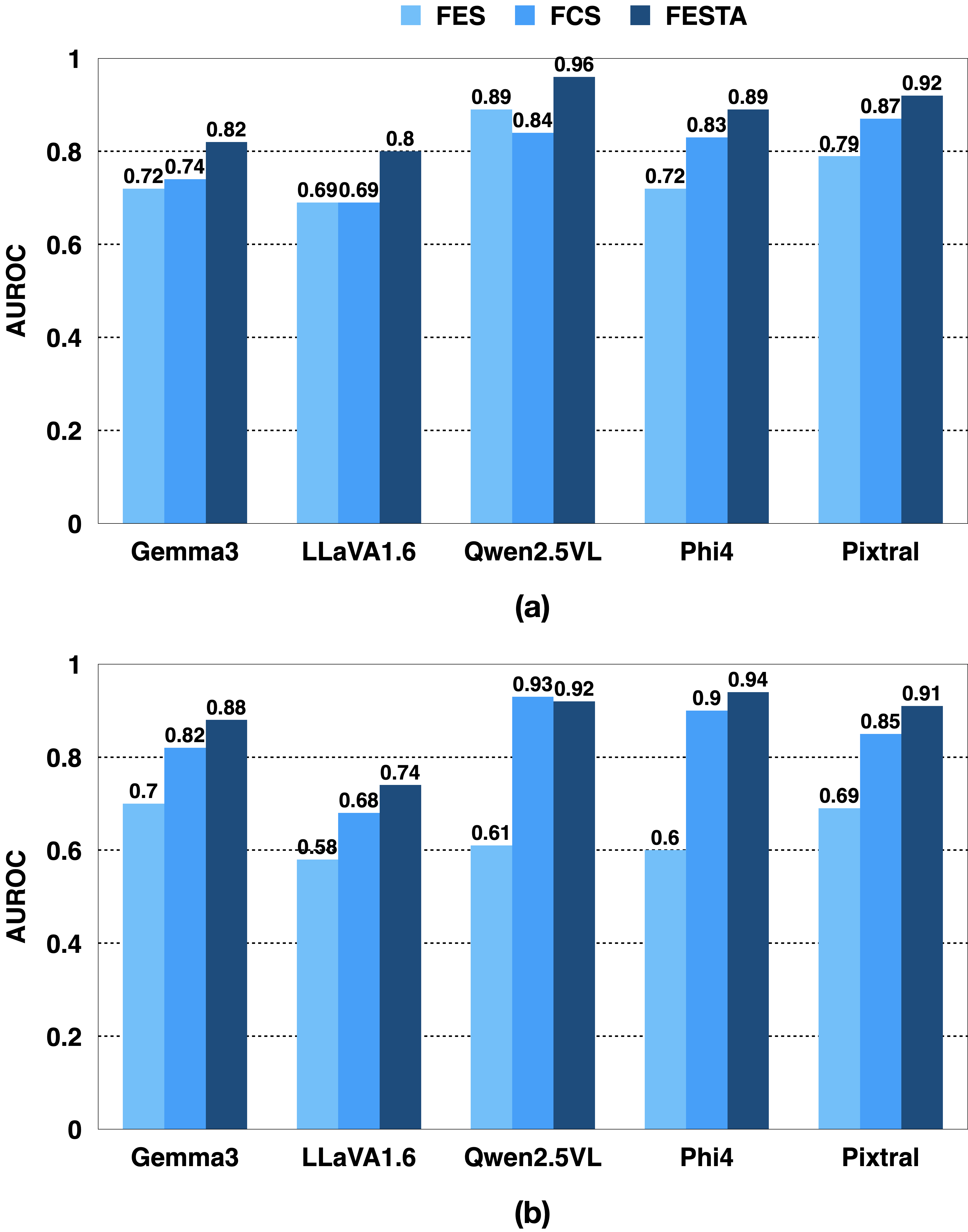}
    \caption{AUROC for uncertainty  based on FES, FCS and \method{} on (a) BLINK, (b) VSR data.
    }
    \label{fig:fes_fcs_barplot_vlm}
\end{figure}
    \item Appendix Figure~\ref{fig:scatter_plots_festa} shows the scatter plot of the confidence  scores ($\frac{1}{\text{uncertainty}}$). Figure~\ref{fig:scatter_plots_oe} shows the scatter plot of the detection scores for the output sampling baseline approach (OE). It is evident that poor baseline AUROCs for different models are primarily caused by low uncertainty mis-predictions, which are majorly detected using \method.
    \item We also compare the effectiveness of our proposed uncertainty quantification method, which uses the KL-divergence distance from the predictive distribution of an ideally consistent and sensitive model. To compare, we compute the AUROCs by replacing the KL-divergence with the standard entropy measure and the results are reported in Appendix Tables~\ref{tab:kl-vs-entropy-audio} and~\ref{tab:kl-vs-entropy-vision}. They highlight the superiority of using the KL distance in a pseudo-supervised fashion, over simply the entropy.
\end{itemize}

\subsection{Number of equivalent samples}
For the vision tasks, the value of $K$ was varied from $8$ to $112$, including FES and FCS samples, while for audio tasks, $K$ was varied from $10$ to $120$ (Appendix Tables~\ref{tab:auc-sample-scaling-blink},~\ref{tab:auc-sample-scaling-vsr},~\ref{tab:auc-sample-scaling-trea-count},~\ref{tab:auc-sample-scaling-trea-duration} and~\ref{tab:auc-sample-scaling-trea-order}). 
The results from this analysis highlight that, even with $K=16$ (vision tasks) and $K=20$ (audio tasks), \method{} can provide a robust estimate of model uncertainty.

\section{Conclusion}
The proposed multimodal uncertainty estimation algorithm, \method{}, is presented as a principled and formally grounded framework for selective prediction by MLLMs. Firstly, it introduces a novel input sampling paradigm based on functional equivalence (FES) and complementarity (FCS), a previously unexplored space, and demonstrates its effectiveness for abstention in the unsupervised and black-box settings. Secondly, it proposes the KL-distance from a hypothetically ideal model as an uncertainty score over the widely used entropy.\\
FESTA enables accurate abstention from incorrect predictions in both audio-LLMs and vision-LLMs, elicited by improved AUROC values.
Given the limited grounding of current MLLMs (especially lightweight MLLMs), they tend to produce biased, low-uncertainty hallucinations. A key contribution of \method{} is its ability to detect and abstain from such hallucinations. This acts as the core reason behind the substantial improvements in selective prediction performance across models.\\
Building on these findings, we plan to extend \method{} to support open-ended generation and multimodal outputs in MLLMs, beyond MCQA tasks.

\section{Acknowledgements}
We thank TCS for supporting this work through the TCS Research Scholar Program and Google for providing compute grants through the GCP Credit Award.

\section{Limitations}
Despite its effectiveness, the current formulation of \method{} has a few limitations and opens several avenues for  future work:
\begin{itemize}[leftmargin=*, itemsep=2pt, topsep=2pt]
    \item \textbf{Computational Overhead:} Unlike standard LLMs, \method{} relies on input samples from both text and non-text modalities, increasing computational demand.  
    This also means that about $K$ additional inferences are made per sample, thereby, increasing the computational demand for computing the uncertainty metric. 
    While the current implementation prioritizes predictive performance—suitable for high-stakes scenarios such as safety-critical applications—reducing latency using FES/FCS samples passed through quantized versions of the models may be undertaken as future research to further reduce the computational demand. 

    \item \textbf{Generation of FES and FCS:} The audio and visual samples for FES and FCS are generated using standard audio/image augmentation tools as well as text rephrasing using LLMs, with limited human intervention. While this process is currently semi-automated with limited human supervision and performed seamlessly for the tasks considered in this work, fully automating the process of FES and FCS generations for more complex audio/image tasks might require additional work.
    
    \item \textbf{Improving LLM Behavior with FES and FCS:} In the current work, the FES and FCS were used to analyze the uncertainty, without attempting to improve the base model performance. However, the dual-space sampling strategy (functional equivalence and complementarity) could also be incorporated as in-context prompts to nudge the LLMs to ensure equivalence/complementarity in the outputs, thereby improving base LLM accuracy. For example, one could input $K$ FES  and prompt the LLM to ensure that the answer to the original prompt should also match the response to all the $K$ FES. With this constraint, the LLM would be forced to generate a consistent response that matches the response to the FES, thereby improving the prediction accuracy. 
    
    \item \textbf{Extension to open-text Generation:}  In the current work, the scope was limited to multi-choice question answering tasks only. Many real-world applications of multimodal LLMs require open-ended, free-form text outputs. A significant future direction is to extend \method{} beyond classification tasks, enabling uncertainty-aware abstention in generative settings. Further, audio and image generation tasks open up new avenues for uncertainty estimation, which is not addressed in this work. 
\end{itemize}


\bibliography{references}

\appendix
\newpage 
\section{Appendix}
\label{sec:appendix}

\subsection{Equivalence of input and FES samples}
\label{sec:equivalence-proof-fes}
\textbf{Proof of Equivalence relation for FES} - 
We show that $\sim_{\mathcal{E}}$ satisfies the three properties of an equivalence relation.

\begin{proposition}
\label{proposition:equivalence}
The relation $\sim_{\mathcal{E}}$ defined over inputs $\mathbf{X} \sim_{\mathcal{E}}$  $\tilde{\mathbf{X}}$  such that $T(\mathbf{X}) = T(\tilde{\mathbf{X}})$ and $M_{\texttt{ideal}}(\mathbf{X}) = M_{\texttt{ideal}}(\tilde{\mathbf{X}})$ is an equivalence relation on the input space.
\end{proposition}
\begin{proof}

\textbf{(Reflexivity):} For any input $\mathbf{X}$, we have
\[
T(\mathbf{X}) = T(\mathbf{X}), \quad M_{\text{ideal}}(\mathbf{X}) = M_{\text{ideal}}(\mathbf{X}),
\]
so $\mathbf{X} \sim_{\mathcal{E}} \mathbf{X}$.

\textbf{(Symmetry):} Suppose $\mathbf{X} \sim_{\mathcal{E}} \tilde{\mathbf{X}}$. Then,
\[
T(\mathbf{X}) = T(\tilde{\mathbf{X}}), \quad M_{\text{ideal}}(\mathbf{X}) = M_{\text{ideal}}(\tilde{\mathbf{X}}).
\]
By symmetry of equality, this implies:
\[
T(\tilde{\mathbf{X}}) = T(\mathbf{X}), \quad M_{\text{ideal}}(\tilde{\mathbf{X}}) = M_{\text{ideal}}(\mathbf{X}),
\]
so $\tilde{\mathbf{X}} \sim_{\mathcal{E}} \mathbf{X}$.

\textbf{(Transitivity):} Suppose $\mathbf{X} \sim_{\mathcal{E}} \tilde{\mathbf{X}}$ and $\tilde{\mathbf{X}} \sim_{\mathcal{E}} \hat{\mathbf{X}}$. Then,
\[
T(\mathbf{X}) = T(\tilde{\mathbf{X}}), \quad M_{\text{ideal}}(\mathbf{X}) = M_{\text{ideal}}(\tilde{\mathbf{X}}),
\]
and
\[
T(\tilde{\mathbf{X}}) = T(\hat{\mathbf{X}}), \quad M_{\text{ideal}}(\tilde{\mathbf{X}}) = M_{\text{ideal}}(\hat{\mathbf{X}}).
\]
By transitivity of equality, we get:
\[
T(\mathbf{X}) = T(\hat{\mathbf{X}}), \quad M_{\text{ideal}}(\mathbf{X}) = M_{\text{ideal}}(\hat{\mathbf{X}}),
\]
so $\mathbf{X} \sim_{\mathcal{E}} \hat{\mathbf{X}}$.

Hence, $\sim_{\mathcal{E}}$ is reflexive, symmetric, and transitive, and thus an equivalence relation.
\end{proof}
\subsection{Equivalence between different FCS samples}
\label{sec:equivalence-proof-fcs}
\begin{proposition}
Let $\mathcal{C}_{\mathbf{X}} := \{ \mathbf{X}' : T(\mathbf{X}') = T(\mathbf{X}), \; M_{\text{ideal}}(\mathbf{X}') \neq M_{\text{ideal}}(\mathbf{X}) \}$ denote the set of functionally complementary samples of input $\mathbf{X}$.

Define a relation $\sim_{\mathcal{C}}$ over this set such that:
\[
\mathbf{X}_1 \sim_{\mathcal{C}} \mathbf{X}_2 \iff T(\mathbf{X}_1) = T(\mathbf{X}_2)
\]
and
\[
 M_{\text{ideal}}(\mathbf{X}_1) = M_{\text{ideal}}(\mathbf{X}_2).
\]
Then \( \sim_{\mathcal{C}} \) is an equivalence relation over the set of functionally complementary samples \( \mathcal{C}_{\mathbf{X}} \).
\end{proposition}

\begin{proof}
We verify the three properties of equivalence:

\textbf{(Reflexivity):}  
For any \( \mathbf{X}_1 \in \mathcal{C}_{\mathbf{X}} \), clearly \( T(\mathbf{X}_1) = T(\mathbf{X}_1) \) and \( M_{\text{ideal}}(\mathbf{X}_1) = M_{\text{ideal}}(\mathbf{X}_1) \), so \( \mathbf{X}_1 \sim_{\mathcal{C}} \mathbf{X}_1 \).

\textbf{(Symmetry):}  
If \( \mathbf{X}_1 \sim_{\mathcal{C}} \mathbf{X}_2 \), then:
\[
T(\mathbf{X}_1) = T(\mathbf{X}_2), \quad M_{\text{ideal}}(\mathbf{X}_1) = M_{\text{ideal}}(\mathbf{X}_2).
\]
By symmetry of equality, the reverse also holds:
\[
T(\mathbf{X}_2) = T(\mathbf{X}_1), \quad M_{\text{ideal}}(\mathbf{X}_2) = M_{\text{ideal}}(\mathbf{X}_1),
\]
so \( \mathbf{X}_2 \sim_{\mathcal{C}} \mathbf{X}_1 \).

\textbf{(Transitivity):}  
If \( \mathbf{X}_1 \sim_{\mathcal{C}} \mathbf{X}_2 \) and \( \mathbf{X}_2 \sim_{\mathcal{C}} \mathbf{X}_3 \), then:
\[
T(\mathbf{X}_1) = T(\mathbf{X}_2) = T(\mathbf{X}_3)
\]
and 
\[
M_{\text{ideal}}(\mathbf{X}_1) = M_{\text{ideal}}(\mathbf{X}_2) = M_{\text{ideal}}(\mathbf{X}_3)
\]
so \( \mathbf{X}_1 \sim_{\mathcal{C}} \mathbf{X}_3 \).

Thus, \( \sim_{\mathcal{C}} \) is reflexive, symmetric, and transitive over \( \mathcal{C}_{\mathbf{X}} \), and therefore an equivalence relation.
\end{proof}
\subsection{FES uncertainty closed-form expression}
\label{sec:fes-uncertainty-closed-form}
\begin{proposition}
Then the KL divergence from the certain model to $q_{\text{FES}}$ simplifies to:
\[
U_{\text{FES}}(M \mid \mathbf{x}) := -\log q_{\text{FES}}(y = \hat{y} \mid \mathbf{x}).
\]
\end{proposition}

\begin{proof}
Recall that KL divergence between distributions $q_{\text{certain}}(y|x)$ and $q_{\text{FES}}(y|x)$ is defined as:
\[
\begin{aligned}
D_{\text{KL}}(q_{\text{certain}}(y|\x) \| q_{\text{FES}}(y|\x)) =\\ \sum_{y} q_{\text{certain}}(y|\x) \log \frac{q_{\text{certain}}(y|\x)}{q_{\text{FES}}(y|\x)}
\end{aligned}
\]

Substituting \( q_{\text{certain}}(y|\x) = \delta_{y, \hat{y}} \), we have:
\begin{align*}
D_{\text{KL}}\left( \delta_{y, \hat{y}} \,\|\, q_{\text{FES}}(y \mid \mathbf{x}) \right)
&= \sum_{y} \delta_{y, \hat{y}} \log \frac{\delta_{y, \hat{y}}}{q_{\text{FES}}(y \mid \mathbf{x})} \\
&= \log \frac{1}{q_{\text{FES}}(y = \hat{y} \mid \mathbf{x})} \\
&= -\log q_{\text{FES}}(y = \hat{y} \mid \mathbf{x})
\end{align*}

\end{proof}
\subsection{FCS for low-uncertainty hallucinations}
\label{sec:fcs-low-uncertainty-hallucinations}
We focus on he functional complementary sampling. We show that a hallucinating model has tendency to not react to the complementary transformations of the input. The proof sketch is provided below, for a single attention head.
\begin{theorem}[Negation Invariance in Hallucinating Attention Blocks]
Consider a single-head attention block over input $\mathbf{X} = [\mathbf{X}_{\mathcal{D}}, \mathbf{X}_{\mathcal{N}}]$ with attention weights
\[
\alpha_{ij} := \mathrm{softmax}_j\left( \frac{Q_i^\top K_j}{\sqrt{d_k}} \right), \quad \mathrm{Attn}_i := \sum_j \alpha_{ij} V_j.
\]

Let negation act as $\mathbf{X}_{\mathcal{D}}^{\mathrm{neg}} := f_{\mathrm{neg}}(\mathbf{X}_{\mathcal{D}}), \mathbf{X}_{\mathcal{N}}^{\mathrm{neg}} := \mathbf{X}_{\mathcal{N}}$.

Under two hallucination models:
\begin{itemize}
    \item[\textbf{(1)}] \textbf{Missing attention:} $\sum_{j \in \mathcal{D}} \alpha_{ij} \ll \sum_{j \in \mathcal{N}} \alpha_{ij}$,
    \item[\textbf{(2)}] \textbf{Over-reliance on prior:} $K_j \approx K^*, V_j \approx V^*$, independent of $\mathbf{X}$,
\end{itemize}
the attention output satisfies:
\[
\mathrm{Attn}_i^{\mathrm{neg}} \approx \mathrm{Attn}_i^{\mathrm{orig}},
\]
and thus the prediction remains unchanged under negation.
\end{theorem}
\subsection{Details of FES and FCS}
\label{sec:fes-fcs-details}
The generic functionally equivalent transforms applied to image data include:
\begin{itemize}[itemsep=0pt, topsep=0pt]
    \item \textbf{Contrast:} Adjusts image contrast to simulate lighting variation.
    \item \textbf{Blur:} Applies slight blurring to the image
    \item \textbf{Noise:} Adding a small amount of pixel-level noise
    \item \textbf{Masking:} Hides small number of random pixels
    \item \textbf{Rotate:} Rotates the image slightly to simulate viewpoint changes.
    \item \textbf{Shift:} Translates the image slightly in space.
    \item \textbf{Greyscale:} Removes color information while preserving structure.
\end{itemize}
The generic functionally equivalent transforms applied to audio data include:
\begin{itemize}[itemsep=0pt, topsep=2pt]
    \item \textbf{Silence:} Adding small duration of silence in between sound events
    \item \textbf{Volume:} Minor adjustment to the volume of different sound events 
\end{itemize} 
The generic functionally equivalent transforms applied to text data include:
\begin{itemize}[itemsep=0pt, topsep=2pt]
    \item \textbf{Rephrase:} Paraphrasing the question such that the meaning remains unchanged
\end{itemize}
Functionally Complementary Transformation for image-text datasets is done by negating the textual question such that the answer changes.

Functionally Complementary Transformation for audio-text datasets is done by negating the audio such that the answer changes. This is task-specific.
\begin{itemize}[itemsep=0pt, topsep=2pt]
    \item \textbf{Count:} Adding new sound events to the original audio
    \item \textbf{Duration:} Replace the longest or shortest sound event in the audio with a sound event not originally present.
    \item \textbf{Order:} Swap the positions of the sound events in the audio
\end{itemize} 
Examples of Functionally Equivalent Transform and Functionally Complementary Transform for both audio and image are given in \ref{fig:fes_fcs_examples}.
\begin{table*}[t!]
\centering
\caption{Summary of Notations Used in FESTA}
\label{tab:notations}
\resizebox{0.65\textwidth}{!}{%
\begin{tabular}{ll}
\toprule
\textbf{Notation} & \textbf{Description} \\
\midrule
$X = [X_O, X_T]$ & Multimodal input (Non-text and text modalities) \\
$Y$ & Ground truth \\
$T(X)$ & Task to be performed to answer $X$ \\
$q(y|X)$ & Predictive distribution of model outputs given $X$ \\
$\hat{y}$ & Model's predicted output, $\hat{y} = \arg\max_{y \in Y} q(y|X)$ \\
$\tilde{X} \sim P_{\text{FES}}(\tilde{X}|X)$ & Functionally equivalent samples (FES) from $X$ \\
$X' \sim P_{\text{FCS}}(X'|X)$ & Functionally complementary samples (FCS) from $X$ \\
$\mathcal{M}_{\text{ideal}}$ & Ideal model achieving task objective perfectly \\
$\mathcal{M}_{\text{consistent}}$ & A perfectly consistent model under FES\\
$\mathcal{M}_{\text{sensitive}}$ & A perfectly sensitive model under FCS \\
$q_{\text{FES}}(y|X)$ & Predictive distribution over FES samples \\
$q_{\text{FCS}}(y|X)$ & Predictive distribution over FCS samples \\
$U_{\text{FES}}$ & Uncertainty estimate from FES \\
$U_{\text{FCS}}$ & Uncertainty estimate from FCS \\
$U_{\text{FESTA}}$ & Combined FESTA uncertainty estimate \\
$S$ & Finite set of output sequences \\
$S'$ & Modified predictive support ($\{\hat{y}, \hat{y}^c\}$) \\
\bottomrule
\end{tabular}%
}
\label{tab:notations-table}
\end{table*}
\subsection{Image-based FCS samples}
\label{appendix:image-fcs-experiment}
\blue{We considered the complementary operation on images instead of on text, for the 'left-of' and 'right-of' questions (20 image-text pairs). The Table~\ref{tab:roc_auc_image_fcs} contains the results on 3 models and on the BLINK dataset. It is interesting to note that image-based FCS may potentially offer further performance gains. Similar experiments on other models and on the VSR dataset will be taken up in future.}
\begin{table}[t!]
\centering
\resizebox{0.45\textwidth}{!}{%
\begin{tabular}{lcc}
\toprule
\textbf{Model} & \textbf{ROC-AUC (text FCS)} & \textbf{ROC-AUC (image FCS)} \\
\midrule
Qwen-2.5VL    & 0.59 & 0.91 \\
LLaVA-1.6  & 0.66 & 0.57 \\
Phi-4     & 0.92 & 0.98 \\
\midrule
\textbf{Mean} & \textbf{0.72} & \textbf{0.82} \\
\bottomrule
\end{tabular}
}
\caption{\blue{Comparison of the performance of image complementary transform-based FCS samples vs text-based FCS samples, from the BLINK dataset.}}
\label{tab:roc_auc_image_fcs}
\end{table}
\subsection{Semantic entropy in MCQA setting}
\label{appendix:semantic-entropy-mcq}
\blue{Semantic Entropy (SE) \cite{kuhn2023semantic, farquhar2024detecting}
for an input $x$ is defined over semantic clusters $c\in\mathcal{C}$ of
free-text outputs $s$ as
\begin{align*}
SE(x)
&= - \sum_{c\in\mathcal{C}} p(c\mid x)\,\log p(c\mid x) \\
&= - \sum_{c\in\mathcal{C}}
\Bigg(\sum_{s\in c} p(s\mid x)\Bigg)
\log\Bigg(\sum_{s\in c} p(s\mid x)\Bigg),
\end{align*}
where (i) $s$ are generated sequences (answers), (ii) $c$ denotes a semantic
cluster of sequences with the same meaning, and (iii) $p(s\mid x)$ is the
model probability of sequence $s$ given $x$.
(All logarithms are natural.)}

\paragraph{MCQ setting:}
\blue{
Suppose there are exactly $K$ fixed answer choices $a_1,\dots,a_K$.
The standard (Shannon) entropy of the predictive distribution over these
choices is}
\[
H(x)\;=\;-\sum_{i=1}^{K} p(a_i\mid x)\,\log p(a_i\mid x).
\]

\blue{In the MCQ scenario, each answer choice is semantically distinct and
explicitly enumerated. Consequently:
\begin{itemize}
  \item The number of semantic clusters equals the number of choices:
        $|\mathcal{C}|=K$.
  \item There is a one-to-one correspondence $c_i \leftrightarrow a_i$.
  \item Hence, for all $i\in\{1,\dots,K\}$, \(p(c_i\mid x)=p(a_i\mid x)\).
\end{itemize}}

\blue{Substituting this identification into the definition of semantic entropy,}
\begin{align*}
SE(x)
&= - \sum_{c\in\mathcal{C}} p(c\mid x)\,\log p(c\mid x) \\
&= - \sum_{i=1}^{K} p(a_i\mid x)\,\log p(a_i\mid x) \\
&= H(x).
\end{align*}

\noindent\textbf{Conclusion:} 
\blue{
In the MCQ setting with fixed choices, semantic
entropy reduces exactly to standard entropy.}
\subsection{Hyperparameters}
FESTA has the minimal number of hyperparameters- only the number of samples to be used. This makes it easily deployable and devoid of heavy tuning.
We have used the same number of equivalent ($K_1$) and complementary ($K_2$) samples ($K = K_1 = K_2$).\\
\textbf{Vision-LLMs}:
For vision LLMs, $K_1 = K_2 = 56$ is used. Within modalities, for each multimodal data point, $14$ image samplings ($K_{11} = 14$) and $4$ text samplings are used ($K_{12} = 4$).
\\
\textbf{Audio-LLMs}:
For vision LLMs, $K_1 = K_2 = 60$ is used. Within modalities, for each multimodal data point, $15$ audio samplings ($K_{11} = 15$) and $4$ text samplings are used ($K_{12} = 4$).
\subsection{Notations}
The symbols and their meanings are noted in Table~\ref{tab:notations}.

\subsection{Model details and License}
\label{appendix:model-cards}
The models below are used as per the suggested guidelines and only for research purposes.\\
\textbf{Gemma-3}: The $12B$ model is used from \url{https://huggingface.co/google/gemma-3-12b-it} with license~\footnote{\url{https://ai.google.dev/gemma/terms}}.
\\
\textbf{LLaVa-1.6}: The $7B$ model is used from \url{https://huggingface.co/llava-hf/llava-v1.6-mistral-7b-hf} with license~\footnote{\url{http://www.apache.org/licenses/LICENSE-2.0}}.
\\
\textbf{Phi-4}: The $5.6B$ model is used from \url{https://huggingface.co/microsoft/Phi-4-multimodal-instruct} with license~\footnote{\url{https://huggingface.co/microsoft/Phi-4-multimodal-instruct/resolve/main/LICENSE}}.
\\
\textbf{Pixtral}: The $12B$ model is used from \url{https://huggingface.co/mistralai/Pixtral-12B-2409} with license~\footnote{\url{http://www.apache.org/licenses/LICENSE-2.0}}.
\\
\textbf{Qwen-2.5VL}: The $7B$ model is used from \url{https://huggingface.co/Qwen/Qwen2.5-VL-7B-Instruct} with license~\footnote{ \url{http://www.apache.org/licenses/LICENSE-2.0}}.
\\
\textbf{Qwen2-Audio}: The $7B$ model is used from \url{https://huggingface.co/Qwen/Qwen2-Audio-7B-Instruct} with license~\footnote{\url{http://www.apache.org/licenses/LICENSE-2.0}}.
\\
\textbf{SALMONN}: The $13B$ model is used from \url{https://github.com/bytedance/SALMONN} with license~\footnote{\url{http://www.apache.org/licenses/LICENSE-2.0}}.
\\
\textbf{Qwen-2}: The $7B$ text-only LLM in Table~\ref{tab:alm-results-temporal-reasoning} is used from \url{https://huggingface.co/Qwen/Qwen2-7B-Instruct} with license~\footnote{\url{http://www.apache.org/licenses/LICENSE-2.0}}.
\subsection{Computation budget and hardware}
We have used $8$ Nvidia RTX A6000 GPU cards for all our experiments.
\subsection{Number of Equivalent and complementary Samples}
The results for varying the number of samples $K$ are given in Tables~\ref{tab:auc-sample-scaling-blink},~\ref{tab:auc-sample-scaling-trea-count},~\ref{tab:auc-sample-scaling-trea-duration},~\ref{tab:auc-sample-scaling-trea-order},and ~\ref{tab:auc-sample-scaling-vsr}. 
\subsection{Usage of AI assitants}
We have used ChatGPT, only restricted to section summarization and for textual rewriting of parts of the paper.
\begin{table*}[h!]
\centering
\begin{tabular}{lrrrrrrrrrrrrrr}
\toprule
\textbf{Sample\_Size} &   \textbf{2*4}  &   \textbf{2*8}  &  \textbf{2*12} &  \textbf{2*16} &  \textbf{2*20} &  \textbf{2*24} &  \textbf{2*28} &  \textbf{2*32} &  \textbf{2*36} &  \textbf{2*40} &  \textbf{2*44} &  \textbf{2*48} &  \textbf{2*52} &  \textbf{2*56} \\
\midrule
\textbf{Gemma-3}  & 0.79 & 0.82 & 0.81 & 0.81 & 0.81 & 0.82 & 0.81 & 0.81 & 0.81 & 0.81 & 0.81 & 0.81 & 0.81 & 0.81 \\
\textbf{LLaVa}   & 0.76 & 0.77 & 0.74 & 0.76 & 0.78 & 0.76 & 0.76 & 0.76 & 0.77 & 0.77 & 0.77 & 0.77 & 0.77 & 0.77 \\
\textbf{Phi4}    & 0.85 & 0.86 & 0.87 & 0.87 & 0.87 & 0.87 & 0.87 & 0.87 & 0.87 & 0.87 & 0.87 & 0.87 & 0.87 & 0.87 \\
\textbf{Pixtral} & 0.90 & 0.90 & 0.89 & 0.90 & 0.90 & 0.91 & 0.90 & 0.91 & 0.90 & 0.90 & 0.90 & 0.90 & 0.90 & 0.90 \\
\textbf{Qwen-2.5-VL}  & 0.92 & 0.92 & 0.92 & 0.92 & 0.93 & 0.93 & 0.92 & 0.93 & 0.92 & 0.93 & 0.92 & 0.93 & 0.93 & 0.93 \\
\bottomrule
\end{tabular}
\caption{AUC performance of different models across increasing sample sizes on BLINK dataset.}
\label{tab:auc-sample-scaling-blink}
\end{table*}

\begin{table*}[t!]
\centering
\begin{tabular}{lrrrrrrrrrrrrrr}
\toprule
\textbf{Sample\_Size} &   \textbf{2*4}  &   \textbf{2*8}  &  \textbf{2*12} &  \textbf{2*16} &  \textbf{2*20} &  \textbf{2*24} &  \textbf{2*28} &  \textbf{2*32} &  \textbf{2*36} &  \textbf{2*40} &  \textbf{2*44} &  \textbf{2*48} &  \textbf{2*52} &  \textbf{2*56} \\
\midrule
\textbf{Gemma-3}  & 0.85 & 0.87 & 0.88 & 0.89 & 0.88 & 0.88 & 0.88 & 0.88 & 0.88 & 0.88 & 0.88 & 0.88 & 0.88 & 0.88 \\
\textbf{LLaVa}   & 0.67 & 0.72 & 0.69 & 0.72 & 0.72 & 0.70 & 0.73 & 0.74 & 0.73 & 0.75 & 0.76 & 0.74 & 0.73 & 0.74 \\
\textbf{Phi4}    & 0.93 & 0.94 & 0.93 & 0.94 & 0.94 & 0.93 & 0.94 & 0.94 & 0.94 & 0.94 & 0.94 & 0.94 & 0.94 & 0.94 \\
\textbf{Pixtral} & 0.88 & 0.89 & 0.91 & 0.90 & 0.91 & 0.90 & 0.90 & 0.90 & 0.90 & 0.91 & 0.91 & 0.91 & 0.91 & 0.91 \\
\textbf{Qwen-2.5-VL}  & 0.91 & 0.88 & 0.93 & 0.92 & 0.92 & 0.92 & 0.91 & 0.93 & 0.91 & 0.91 & 0.92 & 0.92 & 0.92 & 0.92 \\

\bottomrule
\end{tabular}
\caption{AUC performance of different models across increasing sample sizes on VSR dataset.}
\label{tab:auc-sample-scaling-vsr}
\end{table*}

 
\begin{table*}[t!]
\centering
\begin{tabular}{lrrrrrrrrrrrr}
\toprule
\textbf{Sample\_Size} & \textbf{2*5} & \textbf{2*10} & \textbf{2*15} & \textbf{2*20} & \textbf{2*25} & \textbf{2*30} & \textbf{2*35} & \textbf{2*40} & \textbf{2*45} & \textbf{2*50} & \textbf{2*55} & \textbf{2*60} \\
\midrule
\textbf{Qwen2-audio}    & 0.81 & 0.84 & 0.83 & 0.81 & 0.83 & 0.83 & 0.84 & 0.84 & 0.83 & 0.83 & 0.84 & 0.83 \\
\textbf{SALMONN}   & 0.55 & 0.65 & 0.59 & 0.65 & 0.64 & 0.64 & 0.64 & 0.65 & 0.65 & 0.63 & 0.67 & 0.66 \\
\textbf{SALMONN desc+LLM} & 0.80 & 0.81 & 0.80 & 0.80 & 0.80 & 0.82 & 0.81 & 0.81 & 0.81 & 0.81 & 0.81 & 0.81 \\
\bottomrule
\end{tabular}
\caption{AUC performance for the Audio Event Counting task.}
\label{tab:auc-sample-scaling-trea-count}
\end{table*}

\begin{table*}[t!]
\centering
\begin{tabular}{lrrrrrrrrrrrr}
\toprule
\textbf{Sample\_Size} & \textbf{2*5} & \textbf{2*10} & \textbf{2*15} & \textbf{2*20} & \textbf{2*25} & \textbf{2*30} & \textbf{2*35} & \textbf{2*40} & \textbf{2*45} & \textbf{2*50} & \textbf{2*55} & \textbf{2*60} \\
\midrule
\textbf{Qwen2-audio}  &  0.73 & 0.75 & 0.77 & 0.75 & 0.75 & 0.75 & 0.75 & 0.75 & 0.75 & 0.75 & 0.75 & 0.75 \\
\textbf{SALMONN}   & 0.74 & 0.72 & 0.76 & 0.79 & 0.74 & 0.76 & 0.77 & 0.75 & 0.74 & 0.77 & 0.77 & 0.76 \\
\textbf{SALMONN desc+LLM} & 0.79 & 0.80 & 0.80 & 0.80 & 0.79 & 0.80 & 0.79 & 0.79 & 0.79 & 0.79 & 0.79 & 0.80 \\
\bottomrule
\end{tabular}
\caption{AUC performance for the Duration task.}
\label{tab:auc-sample-scaling-trea-duration}
\end{table*}

\begin{table*}[ht]
\centering
\begin{tabular}{lrrrrrrrrrrrr}
\toprule
\textbf{Sample\_Size} & \textbf{2*5} & \textbf{2*10} & \textbf{2*15} & \textbf{2*20} & \textbf{2*25} & \textbf{2*30} & \textbf{2*35} & \textbf{2*40} & \textbf{2*45} & \textbf{2*50} & \textbf{2*55} & \textbf{2*60} \\
\midrule
\textbf{Qwen2-audio}  & 0.91 & 0.92 & 0.91 & 0.91 & 0.91 & 0.91 & 0.91 & 0.91 & 0.91 & 0.91 & 0.91 & 0.91 \\
\textbf{SALMONN}   & 0.81 & 0.83 & 0.83 & 0.86 & 0.84 & 0.86 & 0.84 & 0.85 & 0.86 & 0.86 & 0.86 & 0.86 \\
\textbf{SALMONN desc+LLM} & 0.89 & 0.90 & 0.90 & 0.90 & 0.90 & 0.90 & 0.90 & 0.90 & 0.90 & 0.90 & 0.90 & 0.90 \\
\bottomrule
\end{tabular}
\caption{AUC performance for the Ordering task.}
\label{tab:auc-sample-scaling-trea-order}
\end{table*}

\begin{figure*}[t]
    \centering
    \includegraphics[width=1.0\linewidth, height=0.5\textwidth]{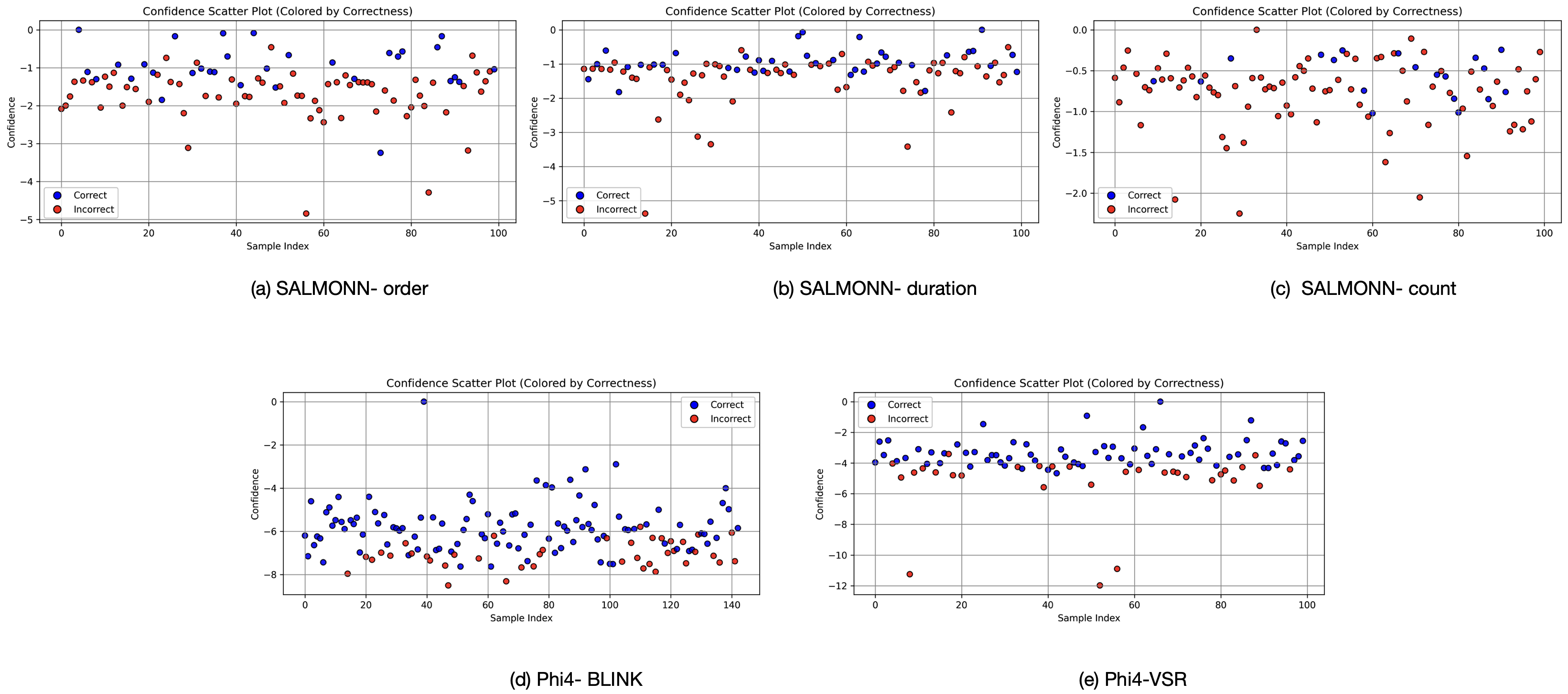}
    \caption{FESTA log(score) plots for best improvement models where score is reciprocal of FESTA uncertainty.
    }
    \label{fig:scatter_plots_festa}
\end{figure*}
\begin{figure*}[t]
    \centering
    \includegraphics[width=1.0\linewidth, height=0.45\textwidth]{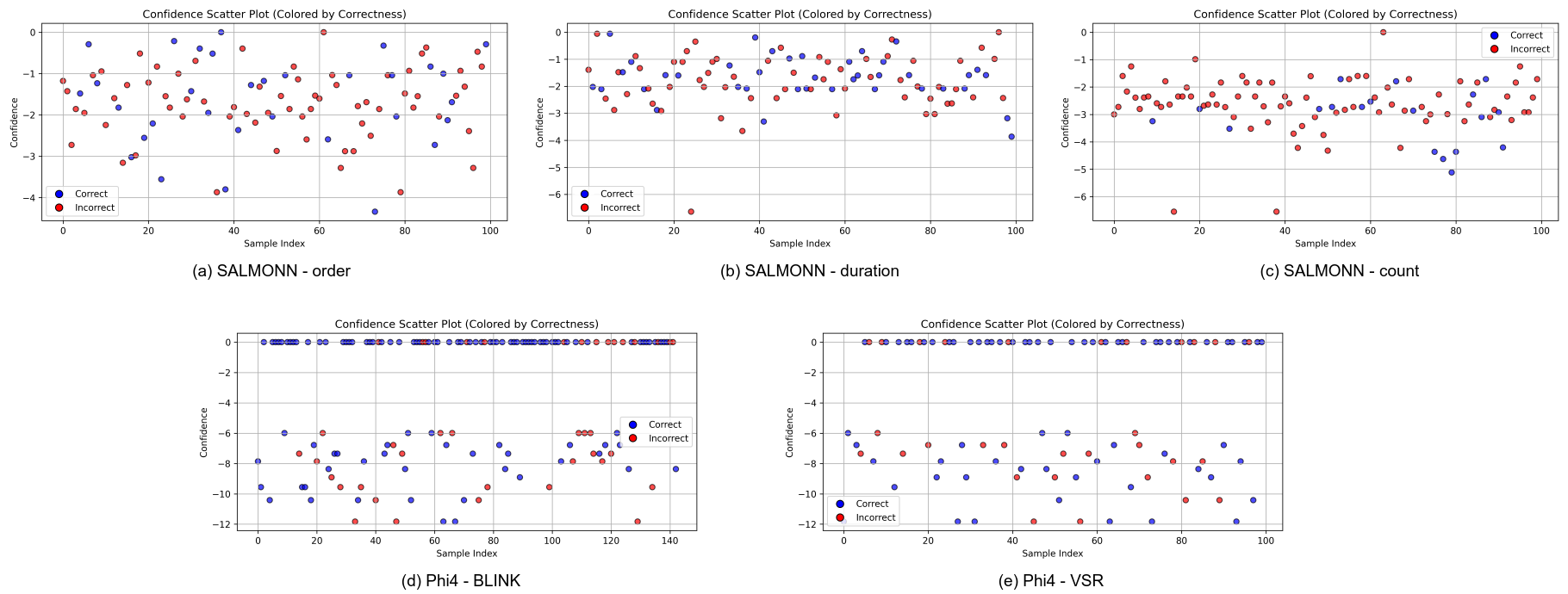}
    \caption{FESTA log(score) plots for output sampling baseline where score is reciprocal of FESTA uncertainty.
    }
    \label{fig:scatter_plots_oe}
\end{figure*}
\begin{table*}[h!]
\centering
\resizebox{0.7\linewidth}{!}{
\def\arraystretch{1.3}
\begin{tabular}{lcccccc}
\toprule
\textbf{Model} & 
\multicolumn{2}{c}{\textbf{TREA-O}} & 
\multicolumn{2}{c}{\textbf{TREA-D}} & 
\multicolumn{2}{c}{\textbf{TREA-C}} \\
\cmidrule(lr){2-3} \cmidrule(lr){4-5} \cmidrule(lr){6-7}
 & Entropy & KL-div & Entropy & KL-div & Entropy & KL-div \\
\midrule
Qwen2-Audio     & 0.59  &  0.91 &  0.67 &  0.75 & 0.38  & 0.83  \\
SALMONN         &  0.58 &  0.86 & 0.60  &  0.76 & 0.27  & 0.66  \\
SAL. des+LLM    & 0.76  & 0.90  &  0.73 & 0.80  & 0.63  & 0.81  \\
\midrule
\textbf{Avg.}   & 0.64  &  0.89 (\textcolor{ForestGreen}{39.1\%}) &  0.67 &  0.77 (\textcolor{ForestGreen}{14.9\%}) &  0.43 &  0.77 (\textcolor{ForestGreen}{79.1\%}) \\
\bottomrule
\end{tabular}
}
\caption{Average performance of audio-LLMs using standard entropy measure compared with the proposed KL-div based measure.}
\label{tab:kl-vs-entropy-audio}
\end{table*}

\begin{table*}[h!]
\centering
\resizebox{0.7\linewidth}{!}{
\def\arraystretch{1.3}
\begin{tabular}{lcccccc}
\toprule
\textbf{Dataset} & \textbf{Gemma3} & \textbf{LLaVA1.6} & \textbf{Qwen2.5VL} & \textbf{Phi4} & \textbf{Pixtral} & \textbf{Avg.} \\
\midrule
BLINK (Entropy)    &    0.57    &     0.66    &      0.79     &    0.65   &      0.73   &    0.68   \\
BLINK (KL-div)   &    0.81    &     0.77    &      0.93     &     0.87  &     0.90    &   0.86 (\textcolor{ForestGreen}{26.5\%})   \\
\midrule
VSR (Entropy)  &    0.61    &     0.55    &     0.56      &     0.41  &     0.58    &    0.54   \\
VSR (KL-div) &    0.88    &    0.74     &     0.92      &  0.94     &    0.91     &    0.88 (\textcolor{ForestGreen}{63.0\%})   \\
\bottomrule
\end{tabular}
}
\caption{Average performance of vision-LLMs using standard entropy measure compared with the proposed KL-div based measure.}
\label{tab:kl-vs-entropy-vision}
\end{table*}
\begin{figure*}[t]
    \centering
    \includegraphics[width=0.8\linewidth, height=1.0\textwidth]{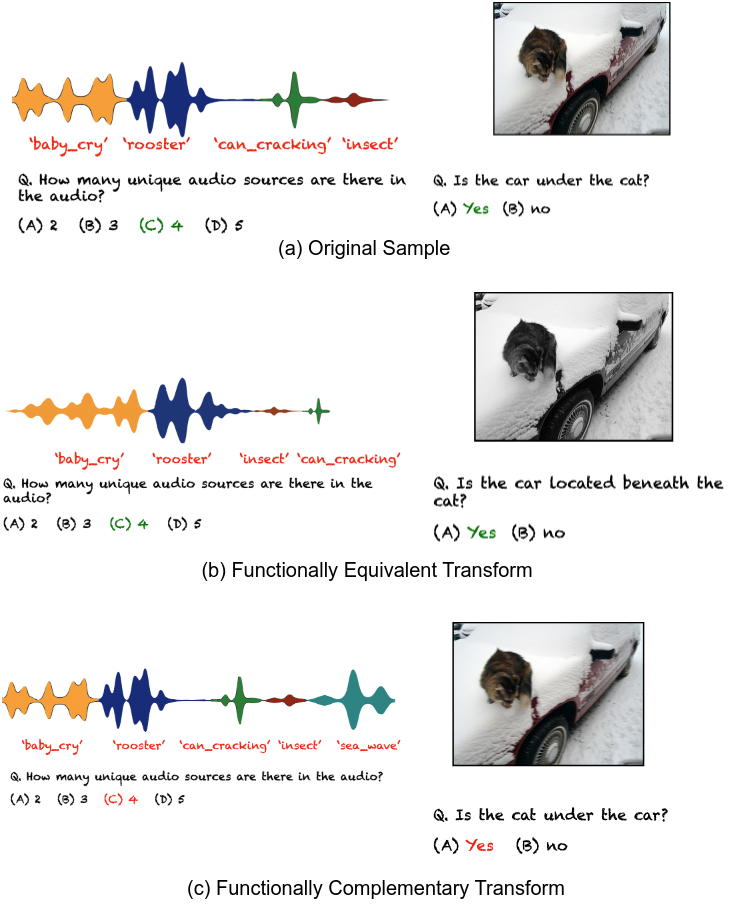}
    \caption{Examples of Functionally Equivalent Transform and Functionally Complementary Transform for both audio-text and image-text questions.
    }
    \label{fig:fes_fcs_examples}
\end{figure*}
\end{document}